\documentclass[11pt, a4paper]{article}

\usepackage[utf8]{inputenc}
\usepackage[T1]{fontenc}
\usepackage{amsmath, amssymb, amsthm, bm}
\usepackage{algorithm}
\usepackage{algpseudocode}
\usepackage[font={footnotesize}]{caption}

\usepackage[utf8]{inputenc}
\usepackage[T1]{fontenc}
\usepackage{lmodern}

\usepackage[a4paper, margin=1in]{geometry}
\usepackage[mathcal]{euscript}
\usepackage{amsmath}    
\usepackage{amssymb}    
\usepackage{amsthm}     
\usepackage{bbm}        
\usepackage{mathtools}  
\usepackage{physics}    
\usepackage{braket}     
\usepackage{dsfont}     
\usepackage{tikz}
\usepackage{tcolorbox}
\usepackage{relsize}	
\usepackage{xparse}  	
\usepackage{MnSymbol}	
\usepackage{authblk}    
\usepackage{graphicx}   
\usepackage{float}      
\usepackage{subcaption} 
\usepackage{xpatch}
\usepackage{setspace}
\usepackage{enumitem}
\usepackage{float}
\usepackage{etoolbox}
\usepackage{mdframed} 	
\usepackage{multicol}
\usepackage[dvipsnames]{xcolor} 
\usepackage{thmtools}
\usepackage{doi}
\usepackage[backend=biber, sorting=nty, style=alphabetic, maxalphanames=4, maxbibnames=4]{biblatex}

\usepackage{hyperref}
\usepackage[capitalize, nameinlink]{cleveref} 
\hypersetup{
	colorlinks=true,
	linkcolor=MidnightBlue,
	citecolor=ForestGreen,
	urlcolor=RoyalBlue
}

\newtheoremstyle{plainroman} 
{}                          
{}                          
{\normalfont}               
{}                          
{\bfseries}                 
{.}                         
{.5em}                      
{}                          

\declaretheorem[name=Theorem, style=plainroman]{theorem}
\declaretheorem[name=Lemma, sibling=theorem, style=plainroman]{lemma}
\declaretheorem[name=Corollary, sibling=theorem, style=plainroman]{cor}
\declaretheorem[name=Proposition, sibling=theorem, style=plainroman]{prop}

\declaretheorem[name=Remark, sibling=theorem, style=plainroman]{remark}

\declaretheorem[name=Problem, sibling=theorem, style=plainroman]{problem}

\theoremstyle{definition}

\newtcolorbox{dbox}[1][]{colback=green!5!white, colframe=white, boxrule=0pt, 
    left=0pt, right=0pt, boxsep=3pt,
    fonttitle=\bfseries, #1}

\newtcolorbox{tbox}[1][]{colback=blue!5!white, colframe=white, boxrule=0pt, 
    left=0mm, right=0mm, boxsep=3pt,
    fonttitle=\bfseries, #1}

\DeclareMathOperator*{\argmin}{argmin}

\renewcommand{\eqref}[1]{Eq.~(\ref{#1})}
\renewcommand{\tr}{\mathrm{Tr}}

\newcommand{\sumin}{\sum_{i=1}^n}
\newcommand{\sumid}{\sum_{i=1}^d}
\newcommand{\qaq}{\quad \text{and} \quad}
\newcommand{\inn}{{i \in [n]}}

\newcommand{\iv}{^{-1}}
\newcommand{\ihf}{^{-\frac12}}
\newcommand{\hf}{^{\frac12}}
\newcommand{\dg}{^{\dagger}}
\newcommand{\ra}{\rangle}
\newcommand{\la}{\langle}

\newcommand{\Lt}{\left}
\newcommand{\Rt}{\right}

\newcommand{\cle}{\mathcal{E}}

\newcommand{\clp}{\mathcal{P}}

\newcommand{\eps}{\varepsilon}

\newcommand{\rb}{\mathrm{B}}
\newcommand{\rf}{\mathrm{F}}

\newcommand{\rmk}{\mathrm{K}}

\newcommand{\bbr}{\mathbb{R}}

\newcommand{\mfg}{\mathfrak{g}}

\newcommand{\bpq}{\mathrm{B}(P,Q)}
\newcommand{\fpq}{\mathrm{F}(P,Q)}

\newcommand{\pdd}{\mathrm{PD}(d)}
\newcommand{\gld}{\mathrm{GL}(d)}
\newcommand{\ld}{\mathrm{L}(d)}

\newcommand{\hrd}{\mathrm{\mathrm{Hr}}(d)}
\newcommand{\tp}{\mathrm{T}_P }
\newcommand{\tpp}{\mathrm{T}_P \mathrm{Pd}(d)}
\newcommand{\lp}{\mathrm{L}_P }

\newcommand{\abi}{[\alpha I, \beta I]}
\newcommand{\abid}{[\alpha' I, \beta' I]}
\newcommand{\clab}{\mathrm{clip}_{\alpha, \beta}}

\newcommand{\isleq}{\overset{?}{\leq}}
\newcommand{\isgeq}{\overset{?}{\geq}}

\newcommand{\lmax}{\lambda_{\max}}
\newcommand{\lmin}{\lambda_{\min}}

\newcommand{\bris}{\mathrm{B}(R_i, S)}

\renewcommand{\grad}{\overline{\nabla}}

\newcommand{\cmu}{\operatorname{comm}(\mathcal{U})}

\usepackage{booktabs}
\usepackage{pifont}
\newcommand{\cmark}{\ding{51}}
\newcommand{\xmark}{\ding{55}}

\declaretheorem[name=Example, sibling=theorem, style=plainroman]{example}

\let\oldtableofcontents\tableofcontents
\renewcommand{\tableofcontents}{{\hypersetup{linkcolor=black}
\oldtableofcontents}}

\title{\vspace*{-2cm}
Projected Riemannian Gradient Descent for the Bures--Wasserstein Barycenter: Dimension-Independent Linear Convergence at Unit Step Size}
\date{}
\author{A. Afham\thanks{Email: \texttt{afham@nus.edu.sg}}}

\affil{Centre for Quantum Technologies,
National University of Singapore, Singapore}

\begin{document}
\emergencystretch 3em

\maketitle
\vspace*{-0.5cm}


\begin{abstract}
The computation of the Bures--Wasserstein (BW) barycenter of an ensemble of positive definite matrices arises throughout machine learning, optimal transport, and quantum information.
Riemannian gradient descent (RGD) at unit step size---the fixed-point iteration used in practice---converges rapidly, yet existing analyses present a dichotomy: unit-step guarantees carry worst-case exponential dependence on the dimension, while dimension-independent guarantees require small step sizes that forfeit the empirical speed.
We resolve this dichotomy, not by improving the guarantees for unit-step RGD, but by proposing a \textit{Projected} RGD algorithm that achieves dimension-independent linear convergence at unit step size.
The achieved rate, $(1 - \kappa^{-3/2})$, where $\kappa$ is the condition number of the ensemble, also polynomially improves on the best small-step guarantee ($\kappa^{3/2}$ versus $\kappa^{5/2}$ iteration complexity).
The crux is a novel \textit{Projection Lemma}: clipping the eigenvalues of a positive matrix to an interval $[\alpha, \beta]$ is the closed-form, non-expansive (1-Lipschitz) BW-metric projection onto the set $\{S : \alpha I \leq S \leq \beta I\}$---a statement which, unlike its known one-sided counterpart, does not follow from convexity.
The projection is moreover free: it reuses an eigendecomposition the next iteration must perform in any case, so the projected and unprojected iterations cost the same per step.
The same analysis covers the invariant matrix projection problem of Brahmachari et al.\ (2025), whose fixed-point algorithm we identify as unit-step RGD on a totally geodesic submanifold, thereby extending the dimension-independent guarantee to that setting verbatim.
\end{abstract}

\section{Introduction}
We consider two optimization problems on the Bures--Wasserstein (BW) manifold of positive definite matrices.
The first is the computation of the BW barycenter of a probability distribution supported on the positive definite cone~\cite{agueh2011barycenters}.
The problem is prevalent in machine learning and optimal transport---averaging covariance matrices, interpolating Gaussian distributions, and processing positive-definite features in computer vision~\cite{agueh2011barycenters, chewi2025statistical, chewi2020gradient, bhatia2019bures}---and features in quantum information through Bayesian quantum tomography~\cite{Afham2022Quantum}, quantum cryptography \& Shannon theory~\cite{Konig2009,tomamichel2015quantum, beigi2013sandwiched, gupta2015multiplicativity}, and quantum resource theories~\cite{winter2016operational, liu2017new, wei2003geometric}.
We first define the \textit{squared} BW distance between positive semidefinite matrices $P, Q \in \mathrm{PSD}(d)$:
\begin{equation}
    \mathrm{B}(P,Q) \equiv \mathrm{d}^2_{\mathrm{BW}}(P, Q) := \Tr[P + Q] - 2 \mathrm{F}(P, Q), \quad \fpq := \Tr[\sqrt{P \hf Q P \hf}] = \Lt\|P \hf Q \hf \Rt\|_1,
\end{equation}
where $\fpq$ is the \textit{quantum fidelity}~\cite{uhlmann1976transition,jozsa1994fidelity} between $P$ and $Q$, and $\|A\|_1 := \Tr[\sqrt{A^\dagger A}]$ is the trace norm; see~\cref{Sec:Preliminaries} for notation and preliminaries.

\begin{problem} \label{Problem1}
    Let $(R_i)_\inn \subseteq \pdd$ be a collection of $d \times d$ positive definite matrices and $w := (w_1, \ldots, w_n)$ an associated probability vector.
    Compute the \textit{BW barycenter} of the distribution $(R_i, w_i)_\inn$:
    \begin{equation}
        S_\star := \argmin_{S \in \pdd} \frac12 \sumin w_i \bris \equiv \argmin_{S \in \pdd} f_1(S).
    \end{equation}
    Throughout, the ensemble is contained in a \textit{spectral interval}: $\alpha I \leq R_i \leq \beta I$ for all $\inn$, for some $0 < \alpha \leq \beta < \infty$, and the ratio $\kappa := \beta/\alpha$ is the ensemble's condition number. For a finite ensemble one may take $\alpha = \min_{\inn} \lmin(R_i)$ and $\beta = \max_{\inn} \lmax(R_i)$, and the same containment covers distributions of infinite support.
\end{problem}
No closed form for the solution is known outside special cases (e.g., pairwise commuting $R_i$), but the optimum satisfies the fixed-point equation $S_\star = \sumin w_i \sqrt{S\hf_\star R_i S\hf_\star}$~\cite{agueh2011barycenters,bhatia2019bures}.

\begin{figure}[t]
    \centering
    \includegraphics[width=\textwidth]{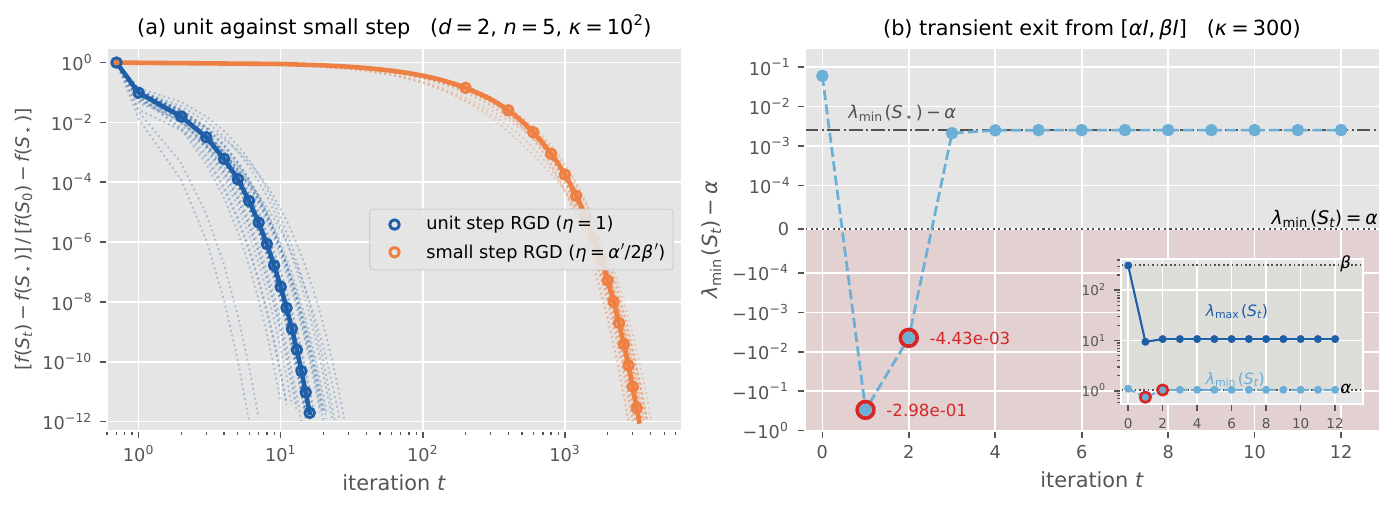}
    \caption{(a) Unit-step against small-step RGD on $30$ instances of the pinned family at $d = 2$, $n = 5$, $\kappa = 100$ (every member with spectrum $\{\alpha, \beta\}$ in a Haar-random eigenbasis), from the initialization $S_0 = \frac{\alpha + \beta}{2} I$: the relative optimality gap per iteration, translucent per instance with the medians opaque.
    At the median, the unit step converges to machine precision in $16$ iterations, the small step $\eta = \alpha'/(2\beta')$ of~\cite{altschuler2021averaging} (marked every $200$th iterate) in about $3300$.
    (b) The transient exit of~\cref{Ex:FeasibleExit} ($\kappa = 300$): the signed deviation $\lmin(S_t) - \alpha$ of the unprojected unit-step iterates, on a symmetric log scale.
    The iterates at $t = 1, 2$ (circled) fall below the floor $\alpha$ (dotted), by $2.98 \times 10^{-1}$ and $4.43 \times 10^{-3}$, before re-entering at $t = 3$ and converging to the strictly interior limit $\lmin(S_\star) - \alpha \approx 2.52 \times 10^{-3}$ (dash-dotted).
    The inset shows the full trajectory of both eigenvalues against the well-conditioned band $\abi:= \{S : \alpha I \leq S \leq \beta I\}$.}
    \label{Fig:Headline}
\end{figure}

\paragraph{An algorithm and a dichotomy.}
The standard method for~\cref{Problem1} is the fixed-point (FP) iteration of~\cite{alvarez2016fixed},
\begin{equation}
    S_{t+1} := S_t \ihf \left[\sum_i w_i \sqrt{S_t \hf R_i S_t \hf}\right]^2 S_t \ihf, \qquad  S_0 := I,
\end{equation}
which coincides with Riemannian gradient descent (RGD) on the BW manifold at unit step size~\cite{zemel2019frechet}: $S_{t+1} = \mathrm{Exp}^\mathrm{BW}_{S_t} [-  \overline{\nabla} f(S_t)]$, where $\overline{\nabla} f$ is the Riemannian gradient and $\mathrm{Exp}^\mathrm{BW}$ is the Riemannian exponential map on the BW manifold.
Empirically, the unit-step algorithm converges rapidly, requires no tuning, and outperforms Euclidean gradient methods, SDP solvers, and RGD with smaller step sizes~\cite{altschuler2021averaging, brahmachari2025fixed}.
Current theory does not match this behavior; instead, it offers a dichotomy.
Analyses of the unit-step algorithm carry an iteration complexity that grows as $\kappa^{d}$ in the worst case~\cite{chewi2020gradient}.
Dimension-free guarantees exist, but only at the \textit{small} step sizes $\eta = O(1/\kappa)$~\cite{altschuler2021averaging},\footnote{Their results allow for a refined condition number $\kappa' := \beta'/\alpha' \leq \kappa$; see~\cref{Rem:RefinedConstants}.} which on typical instances cost two to three orders of magnitude more iterations (\cref{Fig:Headline}(a)).
One may thus have the practical step size or the dimension-free guarantee, but not both; \cref{Sec:RelatedWorks} details both lines.
This article side-steps this dichotomy by presenting a \textit{Projected BW Gradient Descent} algorithm which achieves a dimension-independent linear convergence rate at unit step size: each step of the practical algorithm is composed with an exact, non-expansive projection, and the composition retains the empirical speed of the unit step while admitting the dimension-free guarantee.

The dimension-free analysis of~\cite{altschuler2021averaging} requires the iterates to remain in the \textit{well-conditioned set} $\abi := \{S : \alpha I \leq S \leq \beta I\}$ determined by the ensemble.
The unit-step iterates cannot escape $\abi$ through the \textit{ceiling}:\footnote{Provided the ensemble satisfies $R_i \leq \beta I$, as it does by the choice of $\beta$.} $S \leq \beta I \Rightarrow \mathrm{Exp}_{S}^\mathrm{BW}[- \eta \grad f(S)] \leq \beta I$ for any $\eta \in [0,1]$, a consequence of the convexity of $S \mapsto \sqrt{\lmax(S)}$ along the generalized geodesics traced by the gradient step~\cite{chewi2020gradient,altschuler2021averaging,brahmachari2025fixed}.
However, the floor enjoys no such protection: $S \mapsto \sqrt{\lmin(S)}$ fails to be concave along generalized geodesics---in contrast to its concavity along barycenters~\cite[Thm.~6, Rem.~3]{altschuler2021averaging}---so no step-size-independent lower eigenvalue bound is available.
\cref{Ex:FeasibleExit} (\cref{App:FeasibleExit}) gives an instance ($d = 2$, $n = 3$, $\kappa = 300$) in which the ensemble, the current iterate, and the barycenter all lie in the interior of $\abi$, yet the next two unit-step iterates have minimal eigenvalue strictly below $\alpha$.
Thus at unit step size, the arguments underpinning existing dimension-free rates fail. 
The exit is transient as the ultimate convergence of unit-step RGD to the barycenter is known.
In the aforementioned example, the iterate permanently returned to the well-conditioned set after two steps, converging to a limit that lies strictly above the floor $\alpha$ (see~{\cref{Fig:Headline}(b)}). 
Moreover, such instances had to be adversarially constructed; on random instances the exit never occurred from the canonical initialization.%
\footnote{In randomized trials (\cref{Sec:Numerics}) no exit was ever observed from the prescribed initialization of~\cref{Sec:Algorithms}. From randomly rotated feasible initializations, exits occurred on a fraction $10^{-5}$ to $2 \times 10^{-4}$ of the instances at $n = 3$, the fraction growing as $\kappa$ ranges over $10^2$ to $10^3$. At $n = 5$ no exit was ever observed.}

Such transient exits are avoided in the analysis of~\cite{altschuler2021averaging} by choosing sufficiently small step sizes ($\eta = O(\kappa \iv)$). 
The cost of the small step sizes is demonstrated in \cref{Fig:Headline}(a) on the smallest instructive case, $d = 2$: unit-step RGD gains nearly a decimal digit of accuracy per iteration and reaches the double-precision floor within $16$ iterations at the median, whereas the small step $\eta = \alpha'/(2\beta')$ associated with the dimension-free analysis of~\cite{altschuler2021averaging}\footnote{Strictly, their analysis prescribes the smaller step $\alpha'/(2\sumin w_i \lmax(R_i)) \leq \alpha'/(2\beta')$; every small-step comparison in this article runs at the larger $\alpha'/(2\beta')$, which is conservative in their favor.} needs about 3300 iterations for the same accuracy.
The factor between the two is essentially independent of the target accuracy and across 30 random instances of this family its median is 202 (range 141 to 399).

A step of RGD at any step size evaluates the same map---$n$ transport maps $S_t \iv \# R_i$, each an eigendecomposition---and the step size only scales the resulting displacement, so a unit step and a small step cost the same to compute.
An iteration count is therefore a wall-clock count, and the practitioner who chooses $\eta = 1$ saves two to three orders of magnitude of computation, on every instance.
This is why the unit step is the algorithm used in practice, and why a guarantee for it is highly desired.
As elaborated upon later, we supply this guarantee through a remedy that is standard in constrained optimization: if the iterates do not natively stay in the region of interest, enforce membership by \textit{projection}---provided the projection is cheap and does not undo the progress of the gradient step.
We show that both requirements are met on the BW manifold, at essentially zero additional cost.
{The projection is constructed in~\cref{Sec:ClippingAsProjection}, the full algorithm is stated as~\cref{alg:projected_bwgd} in~\cref{Sec:Algorithms}, and its convergence is proved in~\cref{Sec:Convergence}.}

The second problem we consider is the \textit{invariant matrix projection problem} of~\cite{brahmachari2025fixed}, which generalizes~\cref{Problem1} and unifies certain problems of interest in quantum information~\cite[Sec. 7]{brahmachari2025fixed}.

\begin{problem} \label{Problem2}
    Given a projective unitary representation $\mathcal{U}:= \{U_i\}_\inn \subseteq \mathrm{U}(d)$ and $R \in \pdd$,
    \begin{equation}
        \min_{S \in \clp} \frac12 \mathrm{B}(R, S) \equiv \min_{S \in \clp} f_2(S),
    \end{equation}
    where $\mathcal{P}:= \{P \in \pdd : [P, U_i] = 0 \text{ for all } U_i \in \mathcal{U}\}$ is the set of \textit{invariant positive matrices}.
\end{problem}
\textcite{brahmachari2025fixed} introduced a modified fixed-point iteration to solve the above problem:
\begin{equation}
    S_{t+1} := S_{t} \ihf \left[\Phi \left(\sqrt{S_t\hf R S_t \hf}\right) \right]^2 S_t \ihf, \qquad S_0 = I,
\end{equation}
where $\Phi(X):= \frac1n \sum_i U_i X U_i^\dagger$ is the \textit{twirling map} associated with $\mathcal{U}$.
In~\cite{Afham2025Thesis}, the author showed that this update is unit-step BW gradient descent for the modified objective $\frac{1}{2}\mathrm{B}(R, \Phi(S))$ over the full positive definite manifold; \cref{App:BRT_Algorithm} gives an arguably more elegant interpretation as unit-step BW gradient descent on the \textit{submanifold} of invariant matrices, which we prove is totally geodesic.
\textcite{brahmachari2025fixed} proved linear convergence of their algorithm, but---echoing the barycenter situation---with a worst-case dependence on the dimension.

\begin{remark}
    The invariant matrix projection problem generalizes the BW barycenter problem~\cite{brahmachari2025fixed}.
    We nonetheless analyze the two separately: the barycenter is the more widely used of the two problems, and treating it on its own keeps the constants and the notation simpler.
\end{remark}

\paragraph{Projection is all you need.} In this article we present a \textit{Projected Bures--Wasserstein Gradient Descent} (Projected BW-GD) algorithm: unit-step BW gradient descent composed, at each iteration, with an exact BW-metric projection onto the well-conditioned set $\abi$.
We prove that this algorithm converges linearly at rate $(1 - \kappa^{-3/2})$ for both Problems~\ref{Problem1} and~\ref{Problem2}, with no dependence on the dimension.
To our knowledge this is the first dimension-free guarantee for any unit-step method on these problems; moreover, the iteration complexity $\kappa^{3/2} \log(1/\epsilon)$ improves polynomially on the best known dimension-free guarantee, $O(\kappa^{5/2} \log (1/\epsilon))$ at step size $1/(2\kappa)$~\cite{altschuler2021averaging}.
Note that we do not claim to show the dimension-independent convergence of the FP algorithm.
What we guarantee is that by merely incorporating an eigenvalue-clipping operation at every iteration, we can construct a Projected RGD algorithm, which reduces to the standard FP algorithm whenever the iterates land in $\abi$, and converges at a dimension-independent rate to the barycenter $S_\star$ (see~\cref{Thm:UnifiedConvergence} for a precise statement).
Note that the guarantee comes at no additional cost: clipping reuses an eigendecomposition the following iteration must perform in any case, so the projected and unprojected iterations cost essentially the same per step, except for an $O(d)$ difference (see~\cref{Sec:Algorithms}) which is dominated by the $O(d^3)$ cost of the eigendecompositions.

\subsection{Related Works} \label{Sec:RelatedWorks}
\textbf{The Bures--Wasserstein barycenter and Riemannian gradient descent.}
The BW barycenter has been extensively studied for its role in optimal transport, statistics, and machine learning~\cite{agueh2011barycenters, chewi2025statistical, kroshnin2021statistical}; see also the references therein.
Existence and consistency go back to~\cite{agueh2011barycenters,le2017existence}; statistical properties of empirical barycenters are studied in~\cite{panaretos2019statistical,kroshnin2021statistical,le2022fast}.
The FP algorithm of~\textcite{alvarez2016fixed}, who also established its asymptotic convergence, was identified by~\textcite{zemel2019frechet} as unit-step RGD on the BW manifold, with an independent convergence proof.
Its non-asymptotic analysis has developed along two lines.
\textcite{chewi2020gradient} gave the first linear rate at unit step size via a Polyak--{\L}ojasiewicz (PL) inequality, with constants governed by the conditioning of the iterates; absent control of the smallest eigenvalue, these degrade as $\kappa^{d}$ in the worst case.
\textcite{altschuler2021averaging} then proved a dimension-free rate by establishing eigenvalue control along the trajectory, at the cost of the step-size restriction $\eta \leq 1/(2\kappa)$ and iteration complexity $O(\kappa^{5/2}\log(1/\epsilon))$.
No step-size-independent eigenvalue control is possible as $\{S > 0 : \lambda_{\min}(S) \geq \alpha\}$ is not closed under generalized geodesics~\cite[Rem.~3]{altschuler2021averaging}, and~\cref{Ex:FeasibleExit} realizes the failure inside the feasible set, for two consecutive iterations.
Recent work accelerates the same fixed-point iteration by Riemannian Anderson mixing~\cite{aksenov2026anderson}, with convergence established locally---on suitable BW balls around the solution---which leaves the global, unit-step regime of the practical algorithm without a guarantee.
Our work is best read against this backdrop: rather than seeking finer control of the unprojected dynamics, we restore the invariance of the well-conditioned set by an exact, non-expansive projection.

\textbf{Invariant Matrix Projection.} \textcite{brahmachari2025fixed} introduced the invariant matrix projection problem and a generalized fixed-point algorithm for it, which corresponds to unprojected unit-step RGD on the invariant submanifold (\cref{App:BRT_Algorithm}).
The framework encompasses, among others, the fidelity of asymmetry and of coherence in quantum resource theories, the max-conditional entropy and the order-$1/2$ sandwiched R\'enyi mutual information in quantum Shannon theory, and the geometric measure of entanglement of maximally correlated states~\cite[Sec.~7]{brahmachari2025fixed}.
Their linear convergence rate likewise carries a worst-case dependence on the dimension.

\textbf{Projected and safeguarded methods on the BW manifold.}
Projection steps with respect to the BW distance have appeared in two roles.
As a constraint mechanism, \textcite{junyi2024convergence} analyze projected BW-GD over a \emph{BW ball} $\mathcal{B}(Q; r) := \{S : \mathrm{B}(S, Q) \leq r^2\}$.
Their setting differs from ours in both the feasible set and the role it plays: their BW ball encodes an external constraint imposed on the problem.
Our feasible set is instead a spectral interval chosen for the analysis, and it enjoys three properties their ball does not: it provably contains the unconstrained optimum, projection onto it is closed-form eigenvalue clipping, and the projection is non-expansive.
It is this combination that yields unit-step, dimension-free rates, which are unavailable in their setting.
As a safeguard, one-sided eigenvalue clipping has been used to control iterates in BW-space analyses: \textcite[Prop.~3]{altschuler2021averaging} prove that clipping from above is a BW contraction, and \textcite{lambert2022variational} employ the same device in variational inference.
Our Projection Lemma extends this in two ways: two-sided clipping to $[\alpha, \beta]$ is non-expansive, \emph{and} it is the exact BW-metric projection onto $\abi$.
As explained later, the extension requires a new argument.\footnote{That the one-sided clip is moreover the exact BW projection onto $\{S \geq 0 : S \leq \beta I\}$ is not stated in~\cite{altschuler2021averaging,lambert2022variational}, whose analyses require only the contraction property; it follows readily from their argument in the total-manifold picture, since the lifted constraint set is a unitarily invariant convex ball and Euclidean projections onto such sets descend through the quotient.}
On the structural side, closed-form BW projections onto sets of interest in quantum information---such as sets of bipartite matrices with a given marginal---were recently derived in~\cite{Afham2026Projections}.

\subsection{Our contributions}
The central object of our study is \textit{Projected BW-GD}: unit-step RGD composed, at every iteration, with the BW projection onto $\abi$, where $\alpha, \beta$ are read off the inputs ($\alpha = \min_i \lmin(R_i)$, $\beta = \max_i \lmax(R_i)$ for the barycenter; $\alpha = \lmin(R)$, $\beta = \lmax(R)$ for the invariant projection) and are known to contain the optimum, $S_\star \in \abi$~\cite{altschuler2021averaging, brahmachari2025fixed}.
Our contributions are as follows.

\paragraph{1. A negative result: the well-conditioned set is not invariant at unit step size.}
\cref{Ex:FeasibleExit} certifies that the well-conditioned set is not invariant under the unit step, and the construction extends to every $n \geq 3$ and $d \geq 2$.
This sharpens the geodesic-convexity failure of~\cite[Rem.~3]{altschuler2021averaging}, whose configuration starts outside the well-conditioned set, into a statement about the algorithm itself.
It pinpoints the quantity a fix must control---the spectral floor.

\paragraph{2. The Projection Lemma.}
We prove (\cref{Lem:ProjectionLemma}) that for any positive definite $P$, clipping the eigenvalues of $P$ to $[\alpha, \beta]$ is the unique minimizer of $\mathrm{B}(P, \cdot)$ over $\abi$, and that this projection is non-expansive in the BW distance.
The known one-sided result~\cite[Prop.~3]{altschuler2021averaging} appeals to convexity of the spectral-norm ball and the 1-Lipschitz property of Euclidean projection; the two-sided result (specifically, clipping from below) does not follow from such an argument, and is supplied instead by the Alberti--Uhlmann variational formula~\cite{alberti1983note,Watrous2018Theory}, operator monotonicity of the matrix geometric mean~\cite{BhatiaPD}, and a pinching argument~\cite{tomamichel2015quantum}.
We are not aware of another closed-form BW projection identified as such, and we expect it to serve as a constrained-optimization primitive on this manifold.

\paragraph{3. Dimension-free linear convergence at unit step size, with improved $\kappa$-dependence.}
Because the projection is exact and non-expansive, composing it with the unit gradient step (a) confines all iterates to $\abi$, on which strong convexity and a PL inequality hold with dimension-free constants, and (b) preserves monotone descent.
The result, stated as the unified theorem below and proved as~\cref{Thm:UnifiedConvergence}, is the rate $(1 - \kappa^{-3/2})$ for both problems at unit step size; \cref{tab:comparison} situates it among existing guarantees.
An \emph{a posteriori} certificate (\cref{Cor:Certificate}) bounds the optimality gap of any iterate by $\frac{\kappa^{3/2}}{2}\mathrm{B}(S_t, \rmk(S_t))$, a quantity the iteration has already computed, allowing termination on the fly.

\begin{table}[t]
    \centering
    \begin{tabular}{lccc}
        \toprule
        & Step size & Iteration complexity & Dimension-free \\
        \midrule
        \textcite{chewi2020gradient} & $\eta = 1$ & $\kappa^{d}$-dependent (worst case) & \xmark \\
        \textcite{altschuler2021averaging} & $\eta = \tfrac{1}{2\kappa'}$ & $O({\kappa'}^{5/2} \log \tfrac1\epsilon)$ & \cmark \\
        Brahmachari et al.~\cite{brahmachari2025fixed}$^{\dagger}$ & $\eta = 1$ & $d$-dependent (worst case) & \xmark \\
        \textbf{This work}$^{\dagger}$ & $\eta = 1$ & ${\kappa'}^{3/2} \log \tfrac1\epsilon$ & \cmark \\
        \bottomrule
    \end{tabular}
    \caption{Linear convergence guarantees for first-order methods for the BW barycenter (and, where indicated, the invariant matrix projection problem). Iteration complexity is the number of iterations to reach $f(S_t) - f(S_\star) \leq \epsilon (f(S_0) - f(S_\star))$, up to absolute constants; $d$ is the matrix dimension.
    Two condition numbers appear: $\kappa = \beta/\alpha$ from the spectral bounds of the individual inputs, and the refined $\kappa' = \beta'/\alpha' \leq \kappa$ from the ensemble-average quantities of~\cref{Eq:RefinedConstantsIntro}.
    The refined constant of~\cite{altschuler2021averaging} is $\sumin w_i \lmax(R_i)/\alpha' \geq \kappa'$, so writing their row in $\kappa'$ is conservative in their favor.
    $\kappa$ and $\kappa'$ coincide in the worst case over ensembles.}
    \vspace{2pt}
    {\footnotesize $^{\dagger}$Also covers the invariant matrix projection problem.}
    \label{tab:comparison}
\end{table}

\paragraph{4. A unified square-root-manifold analysis, with structural by-products.}
We identify BW-GD on $\pdd$ with Euclidean GD on the square-root manifold $\gld$ (Burer--Monteiro factored GD~\cite{burer2003nonlinear}).
Euclidean strong convexity of $f$ on $\abi$ then lifts to a PL inequality for $g = f \circ \pi$ with dimension-free constant $\kappa^{-3/2}$, and the descent lemma and the projection combine in three lines to give the rate (\cref{Sec:Convergence}).
This bypasses generalized geodesics and convexity along them, which earlier analyses required~\cite{ambrosio2008gradient, chewi2020gradient, altschuler2021averaging}: their role is played by straight chords of the flat square-root space.
Two by-products follow.
First, the fixed-point algorithm of~\textcite{brahmachari2025fixed} is exactly unit-step RGD on the submanifold of invariant matrices, which we prove is \emph{totally geodesic} (\cref{App:BRT_Algorithm}).
Second, the projected iteration is singular-value-clipped Euclidean GD upstairs (\cref{App:TotalManifold}), under which the barycenter update is an averaging of Procrustes-aligned square roots---the form in which the unit-step descent lemma is a variance identity (\cref{App:UsefulLemmas}).

The following theorem summarizes our main convergence guarantee.
\begin{tbox}
    \begin{theorem}[Unified Linear Convergence of Projected BW-GD]
        Let $f \in \{f_1, f_2\}$ be the objective function with unique optimum $S_\star \in \abi$ and $\kappa = \beta/\alpha$.
        For any feasible initial point $S_0 \in \abi$ (additionally invariant, $S_0 \in \clp$, when $f = f_2$), the iterates $\{S_t\}_{t \geq 0}$ of Projected BW-GD with unit step size satisfy
        \begin{equation}
            f(S_t) - f(S_\star) \le \left( 1 - \frac{1}{\kappa^{3/2}} \right)^t (f(S_0) - f(S_\star)).
        \end{equation}
        For the barycenter this improves to $(1 - \kappa'^{-3/2})$ on the refined interval (\cref{Rem:RefinedConstants}).
    \end{theorem}
\end{tbox}
\begin{proof}
    See~\cref{Sec:Convergence}.
\end{proof}

\paragraph{Refined constants.}
For the barycenter the condition number can be sharpened. Define
\begin{equation} \label{Eq:RefinedConstantsIntro}
    \alpha' := \Big[ \sumin w_i \sqrt{\lmin(R_i)} \Big]^2 \qaq \beta' := \lmax \Big( \sumin w_i R_i \Big),
\end{equation}
which satisfy $\alpha \leq \alpha' \leq \beta' \leq \beta$ and still enclose the optimum, $S_\star \in [\alpha' I, \beta' I]$ (floor by~\cite[Thm.~6]{altschuler2021averaging}, ceiling by~\cite[Thm.~9]{bhatia2019bures}).
Projecting onto $[\alpha' I, \beta' I]$ yields the identical guarantee with $\kappa' := \beta'/\alpha' \leq \kappa$ (\cref{Rem:RefinedConstants}, proved in~\cref{App:RefinedInterval}); this requires additional ingredients, since the $R_i$ need not lie in the refined interval.
Since $\beta' \leq \sumin w_i \lmax(R_i)$, our $\kappa'$ is no larger than the refined condition number of~\cite[Thm.~3]{altschuler2021averaging}, so the rates compare favorably instance-wise.
We work with $\abi$ in the main text, which is simpler and matches the conventions of the literature.

\paragraph{Techniques.}
The analysis splits into an RGD part and a projection part.
For the RGD part, the ingredients---the unit-step descent inequality~\cite{bhatia2019bures,brahmachari2025fixed} and the passage from Euclidean strong convexity to a PL inequality~\cite{chewi2020gradient,junyi2024convergence}---are available in the literature; we route them through the total-manifold identification, under which the Riemannian algorithm becomes a Euclidean gradient method, and the classical PL mechanism~\cite{polyak1963gradient, karimi2016linear} drives the proof.
The PL inequality must come from \emph{Euclidean} strong convexity, for the non-negative curvature of the BW manifold renders the objective geodesically \emph{non}-convex~\cite[App.~B.2]{altschuler2021averaging}; the descent at the full unit step, by contrast, comes from that same curvature, expressed as a variance inequality on the square-root manifold (\cref{Prop:VarianceInequalityBuresDistance}).
Curvature thus takes away geodesic convexity with one hand and returns unit-step descent with the other.
For the projection part, the one-sided argument of~\cite[Prop.~3]{altschuler2021averaging} rests on convexity, which fails for the floor; we establish the two-sided statement with tools from quantum information (\cref{Lem:ProjectionLemma}).
Since the objective is built from squared BW distances to references inside $\abi$---each fixed by the projection---non-expansiveness ensures the projection does not undo the progress of the RGD step, while confining the iterate to the set on which the PL inequality holds.

\paragraph{Organization.}
\cref{Sec:Preliminaries} collects preliminaries.
\cref{Sec:Algorithm} develops the algorithm: the square-root-manifold identification (\cref{Sec:AlgorithmSqrt}), the Projection Lemma (\cref{Sec:ClippingAsProjection}), and the algorithm itself (\cref{Sec:Algorithms}).
\cref{Sec:Convergence} proves the convergence theorem and the a posteriori certificate, and \cref{Sec:Numerics} presents numerical experiments.
The appendices contain the exit instance (\cref{App:FeasibleExit}), the variance-inequality proof of the descent lemma (\cref{App:UsefulLemmas}), the refined-interval guarantee (\cref{App:RefinedInterval}), the total-manifold implementation (\cref{App:TotalManifold}), and the identification of the algorithm of~\textcite{brahmachari2025fixed} as BW-GD on the totally geodesic invariant submanifold (\cref{App:BRT_Algorithm}).

\section{Mathematical Preliminaries} \label{Sec:Preliminaries}
For a positive integer $n$, $[n] := \{1, \ldots, n\}$.
For $d \geq 2$, $\mathrm{L}(d)$ denotes the $d \times d$ complex matrices and $\mathrm{GL}(d)$ the invertible ones.
We equip $\ld$ with the (real) Hilbert--Schmidt (HS, a.k.a.\ Frobenius) inner product $\la A, B \ra := \Re \tr[A^\dagger B]$, where $A^\dagger$ is the conjugate transpose.\footnote{In the settings common in machine learning and statistics the matrices are real (e.g.\ covariance matrices), in which case $A^\dagger$ reduces to the transpose $A^\top$ and the real part in the inner product is redundant. All results of this article hold verbatim in the real case.}
The real vector space of Hermitian matrices is $\hrd := \{A \in \ld :A = A \dg\}$; within $\hrd$ lies the closed cone $\mathrm{PSD}(d)$ of \textit{positive semidefinite} matrices, and its open subset $\pdd$ of \textit{positive definite} matrices
\begin{equation}
    \mathrm{PSD}(d) = \{P \in \hrd : \lambda_i(P) \geq 0 \}, \qaq \mathrm{PD}(d) = \{P \in \hrd : \lambda_i(P) > 0 \},
\end{equation}
where $\lmin(P) \equiv \lambda_1(P) \leq \ldots \leq \lambda_d(P) \equiv \lmax(P)$ are the eigenvalues of $P$ ordered non-decreasingly.
$\mathrm{U}(d) := \{U \in \ld : U \dg U = UU \dg = I_d\}$ denotes the unitary matrices, where $I_d \equiv I$ denotes the $d \times d$ identity matrix.
For $A, B \in \hrd$ we write $A > B$ ($A \geq B$) when $A - B$ is positive (semi)definite.
For $A, B \in \pdd$, define the \textit{matrix geometric mean} as
\begin{equation} \label{Eq:GeometricMean}
    A \# B := A \hf \Lt(A \ihf B A \ihf\Rt) \hf A \hf,
\end{equation}
equivalently the unique $T \in \pdd$ solving the Riccati equation $T A \iv T = B$; it is symmetric, $A \# B = B \# A$, operator monotone in each argument~\cite{BhatiaPD}, and satisfies $A \# (cI) = \sqrt{c}\, A \hf$ for $c > 0$.

\paragraph{The Bures--Wasserstein manifold.}
$\pdd$ is an open subset of the real vector space $\hrd$ and hence a smooth manifold, with tangent space $\tpp \cong \hrd$ at every $P$.
We equip each tangent space with the \textit{BW metric tensor}
\begin{equation}
    \mfg^\mathrm{BW}_P(X, Y) \equiv \langle X, Y \rangle_P^{\mathrm{BW}} = \frac{1}{2} \tr[X \lp(Y)] = \frac12 \Tr [\lp(X) Y],
\end{equation}
where $\lp(X)$ is the unique Hermitian solution of the Lyapunov equation $X = \lp (X) P + P \lp (X)$~\cite{bhatia2019bures,BhatiaPD}.
This metric induces the squared BW distance $\rb(P,Q)$ defined in the introduction; we refer to~\cite{bhatia2019bures,malago2018wasserstein} for proofs of the facts collected here.
The \textit{Riemannian exponential map} is
\begin{equation} \label{Eq:BWExpMap}
    \mathrm{Exp}_P^\mathrm{BW}[X] := \Lt[I + \lp(X)\Rt] \, P \, \Lt[I + \lp(X)\Rt]; \quad \mathrm{dom}(\mathrm{Exp}_P^\mathrm{BW}) = \{X : I + \lp(X)  \in \pdd\}.
\end{equation}
A more illustrative construction of the BW manifold realizes it as a quotient of the manifold of invertible matrices.
Being an open subset of $\ld$, the set $\gld$ is a smooth manifold, which we equip with the Frobenius inner product and distance $\|A - B\|_\mathrm{F}$.
Consider on $\gld$ the unitary equivalence $A \sim B \iff A = BU$ for some $ U \in \mathrm{U}(d)$, together with the \textit{square map}
\begin{equation}
    \pi : \gld \to \pdd, \qquad \pi(A) := A A^\dagger,
\end{equation}
which satisfies $\pi(A) = \pi(B)$ if and only if $A \sim B$; thus $\pdd \equiv \gld/\mathrm{U}(d)$.
The equivalence class of $P \in \pdd$, its \textit{fibre}, is its set of square roots
\begin{equation}
    \pi \iv [P] = \{A \in \gld : A A^\dagger = P\} = \{P \hf U : U \in \mathrm{U}(d)\},
\end{equation}
with $P\hf$ the \textit{principal} square root.\footnote{The quantum information audience would recognize these square roots to be related to the \textit{purifications} of $P$ under the \textit{vec} isomorphism.}
We extend the above notation to sets as well: for $\mathcal{S} \subseteq \pdd$, we define $\pi\iv[\mathcal{S}] := \{A : \pi[A] \in \mathcal{S}\} \subseteq \gld$. 
We call $\gld$ the \textit{total} (or \textit{square-root}) \textit{manifold} and $\pdd$ the \textit{base manifold}.
Since right multiplication by a unitary preserves the Frobenius distance, the distance between fibres is well defined, and it is precisely the BW distance:
\begin{equation} \label{Eq:BuresMin}
    \rb(P, Q) = \min_{A \in \pi \iv [P], \, B \in \pi \iv [Q]} \|A - B\|_\mathrm{F}^2 = \min_{U \in \mathrm{U}(d)} \|P \hf U - Q \hf\|_\mathrm{F}^2 ;
\end{equation}
a pair $(A, B)$ attaining the minimum is an \textit{optimal pairing} for $(P,Q)$.
Expanding the right-hand side and using $\max_{U \in \mathrm{U}(d)} \mathrm{Re} \tr[MU] = \|M\|_1$ recovers the closed form $\rb(P,Q) = \tr[P+Q] - 2 \fpq$.
The construction also recovers the BW metric: it is (up to scaling) the unique metric on $\pdd$ compatible with the Frobenius metric on $\gld$ under $\pi$~\cite{bhatia2019bures}.
For full-rank $P, Q$ (always the case in this work), the optimal unitary in~\cref{Eq:BuresMin} is $U = \mathrm{Pol}(P \hf Q \hf)$, the unitary polar factor (see~\cite{bhatia2019bures}, or~\cite[Lem.~C.1]{Afham2025Riemanniangeometric}); moreover $(P \hf U, Q \hf)$ and $(P \hf, Q \hf U^\dagger)$ are optimal pairings for $(P, Q)$.

\paragraph{Further properties of the BW distance.}
Write $\mathrm{B}_R(\cdot) \equiv \mathrm{B}(R, \cdot)$ for the squared BW distance to a fixed $R \in \pdd$.
Its Euclidean gradient over $\pdd$ is
\begin{equation} \label{Eq:GradBures}
    \nabla \mathrm{B}_R (S) = I - S \iv \# R,
\end{equation}
so the gradient of the barycenter functional is $\nabla f_1(S) = \frac12(I -  \sumin w_i [S \iv \# R_i ])$~\cite{altschuler2021averaging,bhatia2019bures}.
The squared BW distance is \textit{strongly convex} (SC), and in a form that does not require the reference and the argument to share an interval: for every $b > 0$, $\mathrm{B}_R$ is $\frac{\sqrt{\lmin(R)}}{2 b^{3/2}}$-SC on $\{S \in \pdd : S \leq b I\}$~\cite{bhatia2018strong,brahmachari2025fixed}, that is,
\begin{equation} \label{Eq:MixedSC}
    \mathrm{B}_R(Q) \geq \mathrm{B}_R(S) + \la \nabla \mathrm{B}_R(S), Q - S \ra + \frac12 \frac{\sqrt{\lmin(R)}}{2 b^{3/2}} \|Q - S\|_\mathrm{F}^2
\end{equation}
for all $Q, S \leq b I$ (\cref{Lem:MixedSC} records the Hessian form we use in~\cref{App:RefinedInterval}).
Specializing to $R \in \abi$ and $b = \beta$ gives $\lmin(R) \geq \alpha$, so both $f_1$ and $f_2$ are $\mu$-SC on the corresponding $\abi$, where we set once and for all
\begin{equation} \label{Eq:MuF}
    \mu := \frac12 \Lt(\frac{\alpha\hf}{2 \beta^{3/2}}\Rt) = \frac{\alpha\hf}{4 \beta^{3/2}}, \qquad \text{so that} \qquad 4 \mu \alpha = \kappa^{-3/2}.
\end{equation}
The BW distance satisfies the \textit{data processing inequality} (DPI)~\cite{Uhlmann2011Transition, Watrous2018Theory}: for any \textit{quantum channel} (completely positive trace-preserving map) $\Lambda$,
\begin{equation} \label{Eq:BuresDistDPI}
    \bpq \geq \rb(\Lambda(P), \Lambda(Q)).
\end{equation}
We will use the DPI for \textit{pinching channels}~\cite{tomamichel2013thesis}: given mutually orthogonal projectors $\cle := \{E_k\}_{k \in [m]}$ with $E_k E_l = \delta_{kl} E_k$ and $\sum_{k} E_k = I_d$, the associated pinching channel is $\Lambda_\cle (X) := \sum_{k \in [m]} E_k X E_k$.

\paragraph{Preliminaries for the invariant matrix projection problem.}
Let $\mathcal{U} = \{U_i\}_{i=1}^n$ be a projective unitary representation.
Its \textit{commutant} is the unital subalgebra $\cmu := \{X \in \mathrm{L}(d) : [X, U_i] = 0 \text{ for all } i\}$~\cite{Watrous2018Theory}, a vector subspace of $\mathrm{L}(d)$.
Define the invariant sets $\clp := \pdd \cap \cmu$, $\mathcal{H} := \hrd \cap \cmu$, and $\mathcal{G} := \mathrm{GL}(d) \cap \cmu$.
The \textit{twirl} $\Phi(X) := \frac{1}{n} \sumin U_i X U_i\dg$ is a unital quantum channel and is the Hilbert--Schmidt orthogonal projection onto $\cmu$:
\begin{equation}
    \Phi(X) \in \cmu \qaq \la \Phi(X), X - \Phi(X) \ra = 0 \quad \text{for all } X \in \mathrm{L}(d).
\end{equation}
An important geometric distinction: for invariant $P \in \clp$, the tangent space to the full manifold remains all of $\hrd$, whereas the tangent space to the invariant submanifold is exactly $\mathcal{H}$:
\begin{equation}
    \tp \pdd \cong \hrd \qaq \tp \clp \cong \mathcal{H} \qquad \text{for all } P \in \mathcal{P}.
\end{equation}

\section{The Algorithm: Projected Bures--Wasserstein Gradient Descent} \label{Sec:Algorithm}

This section develops the two components of the algorithm and assembles them.
\cref{Sec:AlgorithmSqrt} recalls BW gradient descent and its identification with Euclidean GD on the square-root manifold; \cref{Sec:ClippingAsProjection} proves the Projection Lemma; \cref{Sec:Algorithms} states the algorithms.

\subsection{BW gradient descent and the square-root manifold} \label{Sec:AlgorithmSqrt}

Let $f : \pdd \to \bbr$ be differentiable with Euclidean gradient $\nabla f(P) \in \hrd$, defined by $\mathrm{D}f(P)[X] = \la\nabla f(P), X \ra$.
The \textit{BW (Riemannian) gradient} $\grad f(P)$ is the dual of the derivative with respect to the BW metric, $\la \grad f(P), X \ra_P^\mathrm{BW} = \mathrm{D}f(P)[X] = \la \nabla f(P), X \ra$ for all $X \in \tpp$.
A comparison readily yields a relation between the two gradients:
\begin{equation}
    \nabla f(P) = \frac12 \lp(\grad f(P)) \qaq \grad f(P) = 2 \lp \iv ( \nabla f(P)) = 2 \Lt(P \nabla f(P) + \nabla f(P) P\Rt),
\end{equation}
using $\lp \iv(X) = XP + PX$.
Riemannian gradient descent with step size $\eta > 0$ is therefore
\begin{equation}
    P  \mapsto  \mathrm{Exp}^\mathrm{BW}_P \Lt[- \eta \, \grad f(P)\Rt] =  \Lt[I - 2 \eta \nabla f(P)\Rt]  P  \Lt[I - 2 \eta \nabla f(P)\Rt],
\end{equation}
defined whenever the step lies in the domain of the exponential map, i.e.
\begin{equation} \label{Eq:StepSizeCondition}
    I - 2 \eta \nabla f(P) > 0.
\end{equation}
For both of our objectives (with twirled gradient for $f_2$) the unit step $\eta = 1$ always satisfies this: $I - 2 \nabla f_1(S) = \sumin w_i [S \iv \# R_i]$ and $I - 2 \Phi(\nabla f_2(S)) = \Phi(S \iv \# R)$ are positive definite, being a weighted sum, respectively a twirl, of geometric means of positive definite matrices.

It is instructive to view this update in the square-root manifold.
Let $g := f \circ \pi$ denote the \textit{lifted} function, with $\mathrm{dom}(g) = \gld$; by the chain rule $\nabla g(A) = 2 \nabla f(AA^\dagger) A$.
For any $A \in \pi\iv[P]$, a Euclidean GD step of $g$ satisfies
\begin{equation}
    A \mapsto A - \eta \nabla g(A) = (I - 2 \eta \nabla f(P)) A,
\end{equation}
whose image under $\pi$ is exactly the RGD update above.
Thus \emph{BW gradient descent on $\pdd$ is Euclidean gradient descent on $\gld$}, equivalently Burer--Monteiro factored GD~\cite{burer2003nonlinear, burer2005local}; this observation appears in~\cite[Corollary 5.2]{zheng2025riemannian} and~\cite{maunu2023bures}, and we exploit it throughout. 
The next result expresses the BW length of a gradient step through the lifted gradient; the underlying reason it is a one-line computation is that the fibres of $\pi$ are orbits of an isometry group, so a straight line upstairs that starts at a well-chosen square root remains optimally paired with its starting fibre.
For a more formal version of the above statement, see~\cite[Theorems 3 and 4]{bhatia2019bures}. 

\begin{tbox}
    \begin{lemma}[Step length via the lifted gradient] \label{Lem:StepLength}
        Let $S \in \pdd$ and $A \in \pi \iv [S]$.
        For any $X \in \hrd$ with $M := I - X > 0$, define $A' := MA$ and $S' := \pi(A') = MSM$.
        Then
        \begin{equation}    \label{Eq:StepLengthGeneral}
            \mathrm{B}(S, S') = \|A - A'\|_\rf^2 = \tr[X^2 S],
        \end{equation}
        independently of the choice of $A \in \pi \iv [S]$.
        In particular, let $f : \pdd \to \bbr$ be differentiable, $g := f \circ \pi$, and let $\eta > 0$ satisfy~\cref{Eq:StepSizeCondition} at $S$.
        Taking $X = 2 \eta \nabla f(S)$ gives $A' = A - \eta \nabla g(A)$ and $S' = \mathrm{Exp}^\mathrm{BW}_S[-\eta\, \grad f(S)]$, whence
        \begin{equation} \label{Eq:StepLengthGradient}
            \rb(S, S') = \|A - A'\|_\mathrm{F}^2 = \eta^2\, \|\nabla g(A)\|_\rf^2.
        \end{equation}
    \end{lemma}
\end{tbox}
\begin{proof}
    Since $A - A' = XA$, we have $\|A - A'\|_\rf^2 = \tr[A \dg X^2 A] = \tr[X^2 S]$, independent of the choice of $A$.
    Moreover $\|A - A'\|_\rf^2 = \tr [S + S'] - 2 \Re \la A, A' \ra$, using $S = AA\dg$ and $S' = A'{A'}\dg$, so it suffices to show $\Re \la A, A' \ra = \mathrm{F}(S, S')$, i.e.\ that $(A,A')$ is an optimal pairing.
    Indeed,
    \begin{equation}
        \la A, A'  \ra = \tr[A\dg M A] = \tr [M S] = \mathrm{F}(S, S'),
    \end{equation}
    since $M = S \iv \# S'$ (via the matrix Riccati equation~\cite{BhatiaPD}) and $\mathrm{F}(P, Q) = \tr[P \cdot (P \iv \# Q)]$ for $P, Q \in \pdd$.
    For the particular case, the chain rule gives $\nabla g(A) = 2 \nabla f(S) A$, so $A' = MA = A - \eta \nabla g(A)$ with $M = I - 2 \eta \nabla f(S) > 0$ by~\cref{Eq:StepSizeCondition}.
    Consequently, $S' = MSM = \mathrm{Exp}^\mathrm{BW}_S[-\eta\, \grad f(S)]$, and $\tr[X^2 S] = 4 \eta^2 \tr[\nabla f(S)^2 S] = \eta^2 \|\nabla g(A)\|_\rf^2$, again by the chain rule.
\end{proof}

Thus the BW distance between consecutive BW-GD iterates equals the Euclidean length of the lifted gradient for any square root. 
This lets us monitor progress and certify output purely through gradient norms (\cref{Cor:Certificate}).

\subsection{Eigenvalue clipping is the non-expansive BW projection} \label{Sec:ClippingAsProjection}

We now prove our central technical contribution---the Projection Lemma: clipping the eigenvalues of a positive definite matrix to $[\alpha, \beta]$ yields the closed-form BW-metric projection onto the compact set $\abi$, and this projection is non-expansive (1-Lipschitz).
The feasible set here is a \emph{spectral} one---constraining the whole spectrum to an interval, rather than the trace, the norm, or a distance to a reference---and it is this kind of constraint that arises whenever the conditioning of an iterate must be controlled.
We begin by formally defining the eigenvalue-clipping operation. 
Throughout $0 < \alpha \leq \beta < \infty$, so that clipping maps $\pdd$ into itself. 
We define $\clab: \mathbb{R} \to [\alpha, \beta]$ as $\mathrm{clip}_{\alpha, \beta}(x) := \min\{\max\{\alpha, x\}, \beta\}$, extended to Hermitian matrices spectrally:
\begin{equation}
    \clab(H) := \sumid \mathrm{clip}_{\alpha, \beta} (\lambda_i) v_i v_i \dg,
\end{equation}
where $H = \sumid \lambda_i v_i v_i \dg$ is the eigendecomposition of $H \in \hrd$.
For non-normal matrices we extend the operation through the singular value decomposition, denoted $\clab^{\mathrm{sv}}$.

Before the formal proof, it is worth seeing why the first statement \emph{should} be true.
\textit{Aligning} toward the eigenbasis of $P$ can only help: pinching in that basis fixes $P$, preserves $\abi$, and never increases BW distances, so the minimizer may as well commute with $P$---and once it does, the matrix problem degenerates into $d$ independent scalar problems, each solved by clipping.
The formalization of the above statement into a proof is done by using the data processing inequality and strict convexity of the squared BW distance.
For the non-expansive property, we use the Alberti--Uhlmann characterization of BW distance and operator monotonicity of the matrix geometric mean. 
\begin{tbox}
    \begin{lemma}[Projection Lemma] \label{Lem:ProjectionLemma}
        For any positive definite matrix $P$, the unique BW projection onto the geodesically convex set $\abi$ (for $0<\alpha \leq \beta < \infty$) is given by eigenvalue-clipping:
        \begin{equation}
            \mathrm{clip}_{\alpha, \beta}(P) = \argmin_{Q \in \abi} \rb(P, Q).
        \end{equation}
        Moreover, the projection $\mathrm{clip}_{\alpha, \beta}$ is non-expansive (1-Lipschitz) with respect to the BW distance.
        For all $P, Q \in \pdd$,
        \begin{equation}
            \rb(\mathrm{clip}_{\alpha, \beta}(P), \mathrm{clip}_{\alpha, \beta}(Q)) \le \rb(P, Q).
        \end{equation}
    \end{lemma}
\end{tbox}
\begin{proof}
    We first prove that the projection is given by clipping.
    To this end, observe that the squared BW distance $\rb(P, \cdot)$ is strictly convex over positive definite matrices~\cite{bhatia2018strong}, and $\abi$ is convex and compact, so the projection exists and is unique.
    We claim the projection $Q_\star$ commutes with $P$.
    Let $\Lambda$ be the pinching channel defined by the eigenprojectors of $P$, so $\Lambda(P) = P$.
    For any $Q \in \abi$ we have $\Lambda(Q) \in \abi$ ($\Lambda$ is positive and unital, hence order-preserving), and the DPI gives $\rb(P, \Lambda(Q)) = \rb(\Lambda(P), \Lambda(Q)) \leq \rb(P, Q)$.
    Applied to $Q_\star$, this shows $\Lambda(Q_\star)$ is also a minimizer; uniqueness forces $\Lambda(Q_\star) = Q_\star$, i.e.\ $[Q_\star, P] = 0$.
    The problem therefore reduces, in a common eigenbasis, to the scalar problem
    \begin{equation}
        \min_{Q \in \abi,\, [Q,P]=0} \bpq  =  \min_{\omega_i \in [\alpha, \beta]} \sumid \left(\lambda_i \hf - \omega_i \hf\right)^2,
    \end{equation}
    with $(\lambda_i)_{i \in [d]}$ the eigenvalues of $P$, which is minimized by clipping each $\lambda_i$ to $[\alpha, \beta]$, resulting in $\clab(P)$.
    
    We now prove the projection is 1-Lipschitz. 
    Begin with the Alberti--Uhlmann variational characterization~\cite{alberti1983note, bhatia2019bures}: for $P, Q > 0$,
    \begin{equation} \label{Eq:AlbertiCharacterization}
        \rb(P, Q) = \max_{X > 0} \Lt( \tr[P+Q] - \tr[P X] - \tr[Q X\iv] \Rt) = \max_{X > 0} \Lt(\la P, I - X \ra + \la Q, I - X \iv \ra \Rt),
    \end{equation}
    with optimizer $X = P \iv \# Q$.
    Write $\hat{P} \equiv \mathrm{clip}_{\alpha, \beta}(P)$, $\hat{Q} \equiv \mathrm{clip}_{\alpha, \beta}(Q)$, and let $Y := \hat{P}\iv \# \hat{Q}$ be the optimizer for the pair $(\hat P, \hat Q)$.
    Since $Y > 0$ is a feasible point, we have $\tr[P + Q] - \tr[P Y] - \tr[Q Y\iv] \le \rb(P, Q).$
    It therefore suffices to prove the following:
    \begin{equation} \label{eq:target_inequality} 
        \rb(\hat P, \hat Q) = \la \hat P, I - Y \ra + \la \hat Q, I - Y \iv \ra \isleq  \la P, I - Y \ra + \la Q, I - Y \iv \ra \leq \bpq,
    \end{equation}
    or equivalently we want to show that $\la \hat P - P, Y - I \ra + \la \hat Q - Q, Y \iv - I\ra \isgeq 0$. 
    Via the spectral decomposition $P = \sum_i \lambda_i v_i v_i \dg$, define the gap matrices
    \begin{equation}
        \Delta_P^{\alpha} = \sum_{\lambda_i < \alpha} (\alpha - \lambda_i) v_i v_i \dg \geq 0, \quad \Delta_P^{\beta} = \sum_{\lambda_i > \beta} (\lambda_i - \beta) v_i v_i \dg \geq 0,
    \end{equation}
    so that $\hat{P} - P = \Delta_P^{\alpha} - \Delta_P^{\beta}$, and define $\Delta_Q^{\alpha}, \Delta_Q^{\beta}$ analogously.
    Then the inequality in question (of \cref{eq:target_inequality}) becomes
    \begin{equation} \label{eq:four_terms}
        \tr[\Delta_P^{\alpha} (Y - I)] + \tr[\Delta_P^{\beta} (I - Y)] + \tr[\Delta_Q^{\alpha} (Y\iv - I)] + \tr[\Delta_Q^{\beta} (I - Y\iv)] \isgeq 0.
    \end{equation}
    We now show each term is nonnegative via operator monotonicity of the geometric mean: $R_1 \geq S_1 \text{ and } R_2 \geq S_2 \Rightarrow R_1 \# R_2 \geq S_1 \# S_2$~\cite{BhatiaPD}.
    For the first term: $\hat Q \geq \alpha I$ implies $Y = \hat P \iv \# \hat Q \geq \hat P \iv \# (\alpha I) = \alpha \hf \hat P \ihf$.
    If $v_i$ is an eigenvector of $P$ with eigenvalue $\lambda_i \leq \alpha$, then $v_i$ is an eigenvector of $\hat P$ with eigenvalue $\alpha$, whence
    \begin{equation}
        \la v_i , (Y - I) v_i \ra \geq \alpha \hf \la v_i , \hat P \ihf v_i \ra  - 1 = \sqrt{\alpha/\alpha} - 1  = 0,
    \end{equation}
    and therefore $\tr[\Delta_P^\alpha (Y - I)] = \sum_{i : \lambda_i < \alpha} (\alpha - \lambda_i) \la v_i , (Y - I) v_i \ra \geq 0$.
    The second term is nonnegative by the same argument starting from $\hat Q \leq \beta I$.
    For the last two, observe $Y \iv = (\hat P \iv \# \hat Q) \iv = \hat P \# \hat Q \iv$, and apply the same argument with the roles of $(P, Q)$ exchanged.
\end{proof}
The same operation lifts into the exact Frobenius projection in the total manifold:
\begin{equation}
    \mathrm{clip}^{\mathrm{sv}}_{\sqrt{\alpha}, \sqrt{\beta}}[A] = \argmin_{B \in \pi \iv [\abi]} \| A - B \|_\mathrm{F},
\end{equation}
see \cref{Cor:FrobeniusProjection} in~\cref{App:TotalManifold}.
Unlike its analog in the base BW manifold, the total-manifold Frobenius projection---whose feasible set $\pi\iv[\abi]$ is non-convex---can be \emph{expansive} (see \cref{Rem:NoFrobeniusNonExpansive}).

\begin{remark}[Extension to semidefinite matrices] \label{Rem:PSDExtension}
    \cref{Lem:ProjectionLemma} is stated for positive definite matrices; both assertions extend to all $P, Q \in \mathrm{PSD}(d)$ by continuity.
    The map $\clab$ and the squared BW distance are jointly continuous, so the non-expansiveness inequality passes to the limit directly.
    For the projection property, take $P_\eps := P + \eps I$ for $\eps \searrow 0$: minimality of $\clab(P_\eps)$, i.e.\ $\rb(P_\eps, \clab(P_\eps)) \leq \rb(P_\eps, Q)$ for every $Q \in \abi$, survives the limit $\eps \searrow 0$, so $\clab(P)$---which lifts the zero eigenvalues of $P$ to $\alpha$---remains a BW projection of $P$ onto $\abi$.
    We caution that strict convexity of $\rb(P, \cdot)$ can fail for singular $P$ along directions supported on $\ker P$, so uniqueness of the projection can fail in the semidefinite case; we never need it, as the algorithm only projects positive definite iterates.
\end{remark}

\subsection{Projected Bures--Wasserstein gradient descent} \label{Sec:Algorithms}
A single iteration of Projected BW-GD composes the unconstrained gradient step with the projection.
Throughout we take $\eta = 1$; the unconstrained unit step is then exactly the fixed-point map of~\cite{alvarez2016fixed} (respectively~\cite{brahmachari2025fixed}), which we denote by $\rmk$:
\begin{equation} \label{Eq:ProjectedUpdate}
    S_{t+1} := \clab[\rmk(S_t)], \qquad \rmk(S_t) := \mathrm{Exp}^\mathrm{BW}_{S_t}[- \overline{\nabla} f(S_t)] = [I - 2 \nabla f(S_t)]\, S_t\, [I - 2 \nabla f(S_t)].
\end{equation}
Both instantiations are collected in~\cref{alg:projected_bwgd}; the descent and trapping properties that make the composition work are established in~\cref{Sec:Convergence}.

\noindent
\begin{algorithm}[H]
    \caption{Projected BW Gradient Descent (barycenter and invariant matrix projection)}
    \label{alg:projected_bwgd}
    \begin{algorithmic}[1]
        \Require Barycenter ($f = f_1$): ensemble $\mathcal{R} = (R_i)_{i=1}^n$ and weight vector $w$; invariant projection ($f = f_2$): $R \in \pdd$ and unitary representation $\mathcal{U} = \{U_i\}_{i=1}^n$; iteration count $T$.
        \State $(\alpha, \beta) \gets \begin{cases} \big(\min_i \lmin(R_i),\ \max_i \lmax(R_i)\big) & : f = f_1 \\ \big(\lmin(R),\ \lmax(R)\big) & : f = f_2 \end{cases}$
        \State Initialize $S_0 \in \abi$, additionally $S_0 \in \clp$ when $f = f_2$ (e.g., $S_0 \gets \frac{\alpha+\beta}2 I$)
        \For{$t = 0, 1, \ldots, T-1$}
        \State $\rmk(S_t) \gets \begin{cases}
            S_t\ihf \left( \sum_{i=1}^n w_i \sqrt{S_t\hf R_i S_t\hf} \right)^2 S_t\ihf & : f = f_1 \\[4pt]
            S_t\ihf \left( \Phi\left( \sqrt{S_t\hf R S_t\hf} \right) \right)^2 S_t\ihf & : f = f_2
        \end{cases}$ \Comment{\textit{Step 1: Unit-sized BW-GD Step.}}
        \State $S_{t+1} \gets \clab[\rmk(S_t)]$ \Comment{\textit{Step 2: Projection onto $\abi$ via Clipping.}}
        \EndFor
        \State \Return $S_T$
    \end{algorithmic}
\end{algorithm}
\cref{Sec:Convergence} says how large the iteration count $T$ must be, and~\cref{Cor:Certificate} gives an a posteriori stopping rule that dispenses with fixing $T$ in advance.

\paragraph{The projection is free.}
We note that the projection, via eigenvalue clipping, is essentially free for both problems---it requires nothing more than what the unprojected variant already computes.
Indeed the clipping operation, which constructs $S_{t+1}$ from $\rmk(S_t)$, requires an eigendecomposition of $\rmk(S_t)$.
However, the very next iteration opens by forming $S_{t+1}^{\pm 1/2}$ in order to evaluate $\rmk(S_{t+1})$, which also requires the eigendecomposition of $S_{t+1}$.
Since $S_{t+1} = \clab[\rmk(S_t)]$ is obtained from $\rmk(S_t)$ by acting on its eigenvalues alone, the two share an eigenbasis, and the single decomposition serves both: the projection merely does, one step early, what the unprojected iteration would perform at the start of the very next step.
The clip itself then costs $O(d)$, a comparison per eigenvalue.
Counting spectral decompositions, each of cost $O(d^3)$: an iteration performs $n$ of them for the geometric means $\sqrt{S_t^{1/2} R_i S_t^{1/2}}$ and one for the new iterate, whether or not it is projected---so both variants cost $n + 1$ per step, and the clip contributes only the $O(d)$ comparisons.
Hence the iteration counts reported in~\cref{Sec:Numerics} may be read as running times.
For clarity of exposition, \cref{alg:projected_bwgd} is stated without this reuse; the implementation accompanying the article realizes it.

\paragraph{The iteration on the total manifold.}
The base--total equivalence also lets the whole iteration run upstairs: from any $A_0 \in \pi \iv [S_0]$, singular-value-clipped Euclidean GD of $g$ reproduces the projected iteration exactly, returning to $\pdd$ only upon convergence (\cref{App:TotalManifold}).
We record this not for its computational content but for its structure: upstairs, the update is an averaging of Procrustes-aligned square roots---the form in which the unit-step descent lemma is a variance identity and Euclidean strong convexity becomes the PL inequality that drives~\cref{Sec:Convergence}.

\section{Dimension-Independent Linear Convergence at Unit Step-Size} \label{Sec:Convergence}
We may now prove the unified convergence theorem, and to this end, we clarify the notation.
Throughout this section, $f \in \{f_1, f_2\}$, with $(\alpha, \beta)$ read off the inputs as in~\cref{alg:projected_bwgd} and
\begin{equation}
    \hat \nabla g :=
    \begin{cases}
        \nabla g &  : f = f_1  \\
        \Phi \circ \nabla g &  : f = f_2.
    \end{cases}
\end{equation}
Moreover, the domain of interest is $\abi$ (and $\pi \iv [\abi]$) when $f = f_1$ and $\abi \cap \mathcal{P}$ (and $\pi \iv [\abi] \cap \mathcal{G}$) when $f = f_2$. 
In both cases, $f$ is $\mu$-SC on $\abi$ with $\mu$ as in~\cref{Eq:MuF}.
The optimum $S_\star$ exists, is unique, and lies in the corresponding domain of interest: existence and uniqueness are classical for the barycenter~\cite{agueh2011barycenters} and follow from strong convexity on the compact feasible region for the invariant projection, while the spectral inclusion $S_\star \in \abi$ is established in~\cite{altschuler2021averaging,brahmachari2025fixed}.

The first ingredient is a PL inequality, obtained by transferring strong convexity through the lift of~\cref{Sec:AlgorithmSqrt}, a fact previously noted in~\cite[Lemma 3.4]{junyi2024convergence}; we state it in a form covering both of our problems at once, with the twirl $\Phi$ taken to be the identity map for the barycenter problem ($\mathcal{U} = \{I\}$, so $\cmu = \mathrm{L}(d)$, $\clp = \pdd$, and $\mathcal{G} = \gld$).
\begin{tbox}
    \begin{lemma}[PL inequality for the lifted problem] \label{Lem:PLIneqfromSC}
        Let $0 < a \leq b < \infty$, $\mathcal{U}$ be a projective unitary representation with twirl $\Phi$, and $f$ be $\mu_0$-strongly convex on $[a I, b I]$ with $S_\star = \argmin_{S \in \clp \cap [a I, b I]} f(S)$.
        Then $g := f \circ \pi$ satisfies, for every $A \in \pi \iv \big[[a I, b I]\big] \cap \mathcal{G}$,
        \begin{equation}
            g(A) - g(A_\star) \le \frac{1}{2(4 \mu_0 a)} \| \Phi(\nabla g(A)) \|_{\mathrm{F}}^2,
        \end{equation}
        where $A_\star \in \pi \iv [S_\star] \cap \mathcal{G}$.
        In particular, for the barycenter functional $f_1$ the twirl is trivial ($\mathcal{U} = \{I\}$), and $g_1 := f_1 \circ \pi$ satisfies $g_1(A) - g_1(A_\star) \le \frac{1}{2(4\mu_0 a)} \|\nabla g_1(A)\|_\rf^2$ on $\pi\iv[[a I, b I]]$.
    \end{lemma}
\end{tbox}
\begin{proof}
    Let $S := \pi(A) \in \clp \cap [a I, b I]$.
    For any $P \in \clp \cap [a I, b I]$, strong convexity gives
    \begin{equation} \label{Eq:54}
        f(P) \geq f(S) + \la \nabla f(S), P - S \ra + \frac{\mu_0}{2} \|P - S\|_\mathrm{F}^2.
    \end{equation}
    Since $P - S \in \mathcal{H}$ and $\Phi$ is the self-adjoint idempotent HS-projection onto $\cmu$, we may replace $\nabla f(S)$ by $G := \Phi(\nabla f(S))$ in the inner product.
    Writing $Z := P - S$ and completing the square,
    \begin{equation}
        \la G, Z \ra + \frac{\mu_0}{2}\|Z\|_\mathrm{F}^2 = \frac{\mu_0}{2}\Big\| Z + \frac{G}{\mu_0}\Big\|_\mathrm{F}^2 - \frac{1}{2\mu_0}\|G\|_\mathrm{F}^2  \geq  - \frac{1}{2\mu_0}\|G\|_\mathrm{F}^2 .
    \end{equation}
    Substituting into~\cref{Eq:54} and choosing $P = S_\star$,
    \begin{equation} \label{Eq:InvariantPLBase}
        f(S) - f(S_\star) \le \frac{1}{2 \mu_0} \|\Phi(\nabla f(S))\|_\mathrm{F}^2.
    \end{equation}
    To lift, note $\Phi(\nabla g(A)) = 2 \Phi(\nabla f(S)) A$, using $\Phi(XA) = \Phi(X)A$ for $A \in \mathcal{G}$; hence
    \begin{equation}
        \|\Phi(\nabla g(A))\|_\mathrm{F}^2 = 4 \tr\Big( \Phi(\nabla f(S))^2 S \Big) \geq 4 a \|\Phi(\nabla f(S))\|_\mathrm{F}^2,
    \end{equation}
    since $S \geq a I$.
    Combining with~\cref{Eq:InvariantPLBase} and $g(A) = f(S)$, $g(A_\star) = f(S_\star)$ gives the claim.
\end{proof}

The special case of~\cref{Lem:PLIneqfromSC} for $f = f_2$ appears in~\cite[Eq.\ 58]{brahmachari2025fixed}.
We broadly follow the proof idea used there and in~\cite{karimi2016linear}.
The lemma as stated is strictly more general: it covers both problems at once and applies to any strongly convex $f$, with the setting of~\cite{brahmachari2025fixed} recovered on choosing $f = f_2$.
In~\cref{App:BRT_Algorithm} we explain that the \textit{twirled} gradient, rather than the gradient, appears because (i) in the base-manifold perspective, it is the BW-orthogonal projection of $\nabla f_2(P)$ onto the invariant tangent subspace $\mathcal{H} \equiv \tp \mathcal{P} \subseteq \tpp$ and (ii) in the total-manifold perspective, it is the (Hilbert--Schmidt) orthogonal projection of the gradient $\nabla g(A)$ to the invariant subspace $\cmu$ (of $\mathrm{L}(d)$).

The second ingredient is the descent half of the argument, supplied by the following proposition; combined with the PL inequality above, it yields the linear rate.

\begin{tbox}
    \begin{prop}[Projected descent lemma] \label{Prop:ProjectedDescent}
        Let $f \in \{f_1, f_2\}$ and $g := f \circ \pi$ be the lifted functional. 
        Let the current \textit{square-root} iterate be
        \begin{equation} \label{Eq:ProjLemmaAt}
            A_t \in 
            \begin{cases}
                \pi \iv [\abi] & : f = f_1 \\
                \pi \iv [\abi] \cap \mathcal{G} & : f = f_2 \\
            \end{cases}, 
        \end{equation} 
        $S_t = \pi (A_t)$, and set $A_{t+1} = \mathrm{clip}^{\mathrm{sv}}_{\sqrt\alpha, \sqrt\beta}[A_t - \hat \nabla g(A_t)]$ and $S_{t+1} = \pi(A_{t+1}) = \clab[\rmk(S_t)]$.
        Then:
        \begin{enumerate}[label=(\roman*), leftmargin=*]
            \item {(Feasibility.)} $S_{t+1} \in \abi$ and $A_{t+1} \in \pi \iv [\abi]$; when $f = f_2$, additionally $S_{t+1} \in \clp$ and $A_{t+1} \in \mathcal{G}$.
            \item {(Descent.)} $f(S_{t+1}) \leq f(S_t) - \frac12 \mathrm{B}(S_t, \rmk(S_t))$, equivalently, upstairs,
            \begin{equation}
                g(A_{t+1}) \leq g(A_t) - \frac12 \|\hat \nabla g(A_t)\|_{\mathrm{F}}^2.
            \end{equation}
        \end{enumerate}
    \end{prop}
\end{tbox}
\begin{proof}
    Assume $A_t$ satisfies~\cref{Eq:ProjLemmaAt} and $S_t := \pi(A_t)$.
    We first verify that feasibility is preserved.
    Membership $S_{t+1} \in \abi$ holds by construction of the clipping map.
    For $f_2$, invariance is also preserved.
    Upstairs,
    \begin{equation}
        A_t \in \mathcal{G} \Longrightarrow A_t - \Phi( \nabla g_2(A_t)) \in \mathcal{G}.
    \end{equation}
    Recall that $\mathcal{G} := \cmu \cap \gld$.
    The inclusion (of the unclipped next iterate) in $\cmu$ follows since both terms lie in $\cmu$ ($\Phi$ maps into $\cmu$) and $\cmu$ is a vector subspace of $\mathrm{L}(d)$.
    For invertibility, note that $\Phi(\nabla g_2(A_t)) = 2 \Phi(\nabla f_2(S_t)) A_t$, using $\Phi(XA) = \Phi(X) A$ for $A \in \mathcal{G}$; hence $A_t - \Phi(\nabla g_2(A_t)) = [I - 2 \Phi(\nabla f_2(S_t))] A_t$, and $I - 2\Phi(\nabla f_2(S_t)) = \Phi(S_t \iv \# R) > 0$ (see~\cref{Eq:SubmanifoldUpdateRule,Eq:BRTFixedPoint}) while $A_t$ is invertible by assumption.
    On the base manifold, the unconstrained step $\rmk(S_t) = \Phi(S_t\iv \# R)\, S_t\, \Phi(S_t\iv \# R)$ is a product of elements of $\cmu$, and $\cmu$ is an algebra; hence $\rmk(S_t) \in \clp$.
    Since $\clab$ acts spectrally, we conclude $S_{t+1} \in \abi \cap \clp$.
    Finally, to see that $A_{t+1} \in \mathcal{G}$, write $A'_t := A_t - \Phi(\nabla g_2(A_t)) \in \mathcal{G}$ and note that singular value clipping preserves the polar factor (it modifies only the singular values), so that
    \begin{equation}
    \begin{aligned}
        A_{t+1} &= \mathrm{Pol}(A'_t) \cdot  \mathrm{clip}_{\sqrt\alpha, \sqrt\beta}(|A'_t|) \\
        &=  (A'_t |A_t'| \iv)  \cdot \mathrm{clip}_{\sqrt\alpha, \sqrt\beta}(|A'_t|) \in \cmu,
    \end{aligned}
    \end{equation}
    where the last inclusion follows from the facts that every factor in the product is an element of $\cmu$---for $A'_t$ by the previous step; for $|A'_t| = (A_t'^\dagger A_t')\hf$ and its clipped version because, for Hermitian $Z \in \cmu$ and any function $h$ applied spectrally, $U_i\, h(Z)\, U_i\dg = h(U_i Z U_i\dg) = h(Z)$, so $h(Z) \in \cmu$; and for $|A'_t|\iv$ since the inverse of an invertible element of $\cmu$ lies in $\cmu$---and that $\cmu$ is an algebra and hence it is closed under products. 
    As the clipped singular values lie in $[\sqrt\alpha, \sqrt\beta]$, $A_{t+1}$ is invertible; hence $A_{t+1} \in \cmu \cap \gld = \mathcal{G}$.    

    Because the reference matrices---the ensemble $(R_i)_{i\in[n]}$ for $f_1$, the target $R$ for $f_2$---lie in $\abi$ and are fixed by $\clab$, the non-expansiveness of the BW projection guarantees that the objective does not increase upon projecting the unconstrained step $\rmk(S_t)$:
    \begin{equation}
        f(S_{t+1}) = f(\clab[\rmk(S_t)]) \leq f(\rmk(S_t)).
    \end{equation}
    Furthermore, the objective value at the unconstrained step satisfies the descent inequality $f(\rmk(S_t)) \leq f(S_t) - \frac12 \mathrm{B}(S_t, \rmk(S_t))$, established for $f_1$ in~\cite[Theorem 10]{bhatia2019bures} (an alternative proof of which is provided as~\cref{Prop:VarianceInequalityBuresDistance} in the Appendix) and for $f_2$ in~\cite[Eq.~37]{brahmachari2025fixed}.
    Combining these two inequalities yields the descent lemma for the projected sequence in the base manifold:
    \begin{equation}
        f(S_{t+1}) \leq f(S_t) - \frac12 \mathrm{B}(S_t, \rmk(S_t)).
    \end{equation}
    
    To establish the equivalence in the square-root manifold, recall that $f \circ \pi = g$, yielding $f(S_t) = g(A_t)$; moreover, since eigenvalue clipping on the base manifold is equivalent to singular value clipping on the square-root manifold, $\pi \circ \mathrm{clip}^{\mathrm{sv}}_{\sqrt\alpha, \sqrt\beta} = \clab \circ \pi$~(\cref{Cor:FrobeniusProjection}), we have $\pi(A_{t+1}) = S_{t+1}$ and hence $f(S_{t+1}) = g(A_{t+1})$.
    It remains to establish that $\mathrm{B}(S_t, \rmk(S_t)) = \| \hat \nabla g(A_t)\|^2_\rf$, which follows on applying~\cref{Lem:StepLength} with
    \begin{equation}
        X = \hat \nabla g(A_t) A_t \iv =
        \begin{cases}
            2 \nabla f_1(S_t) & :  f = f_1\\
            2 \Phi (\nabla f_2 (S_t))  &: f = f_2
        \end{cases}.
    \end{equation}
    The choice of $X$ is admissible for~\cref{Lem:StepLength}: it is Hermitian, and
    \begin{equation}
        I - X =
        \begin{cases}
            \sumin w_i \, [S_t\iv \# R_i] & : f = f_1 \\
            \Phi(S_t \iv \# R) & : f = f_2
        \end{cases}
    \end{equation}
    is positive definite in both cases, being a weighted sum (respectively, the twirl) of geometric means of positive definite matrices.
    Since $\hat \nabla g(A_t) = X A_t$, the unconstrained Euclidean step satisfies $A_t - \hat\nabla g(A_t) = (I - X) A_t \in \pi\iv[\rmk(S_t)]$, and~\cref{Eq:StepLengthGeneral} yields
    \begin{equation}
        \mathrm{B}(S_t, \rmk(S_t)) = \|X A_t\|_\rf^2 = \|\hat \nabla g(A_t)\|_\rf^2.
    \end{equation}
    Combining the above, the descent inequality is equivalently expressed as $g(A_{t+1}) \leq g(A_t) - \frac12 \|\hat \nabla g(A_t)\|_\rf^2$.
    This concludes the proof.
\end{proof}
We now prove that the projected BW-GD algorithm achieves dimension-independent linear convergence at unit step-size.
\begin{tbox}
    \begin{theorem}[Dimension-independent linear convergence] \label{Thm:UnifiedConvergence}
        Let $f \in \{f_1, f_2\}$ and $S_0 \in \abi$ (additionally, $S_0 \in \mathcal{P}$ if $f = f_2$).
        The iterates of~\cref{alg:projected_bwgd} satisfy
        \begin{equation}
            f(S_t) - f(S_\star) \leq \left(1 - \frac{1}{\kappa^{3/2}}\right)^t (f(S_0) - f(S_\star)).
        \end{equation}
    \end{theorem}
\end{tbox}
\begin{proof}
    The proof is obtained through a combination of~\cref{Prop:ProjectedDescent} (the projected descent lemma) and~\cref{Lem:PLIneqfromSC} (PL inequality from strong convexity). 
    Both are available at every step, as we now check by induction.
    Feasibility propagates along the iteration by~\cref{Prop:ProjectedDescent}(i):
    \begin{equation}
        A_t \in \pi \iv [S_t] \text{ and } S_t \in \abi \quad \Rightarrow \quad A_{t+1} \in \pi \iv [S_{t+1}] \text{ and } S_{t+1} \in \abi,
    \end{equation}
    and, when $f = f_2$, so does invariance: $A_t, S_t \in \cmu \Rightarrow A_{t+1}, S_{t+1} \in \cmu$.
    The hypotheses hold at $t = 0$, since $S_0 \in \abi$ and $A_0 \in \pi \iv [S_0]$ by assumption; and for $f = f_2$ we may take $A_0 = S_0 \hf$, which lies in $\cmu$ because $S_0 \in \clp$ and $\cmu$ is closed under the functional calculus of its elements (see the proof of~\cref{Prop:ProjectedDescent}).
    Induction therefore gives $S_t \in \abi$ and $A_t \in \pi \iv [\abi]$ for every $t \geq 0$, together with $A_t, S_t \in \cmu$ when $f = f_2$.
    Thus we may invoke both the projected descent lemma and the PL inequality at every step.
    The descent lemma gives
    \begin{equation}
        g(A_{t+1}) \leq g(A_t) - \frac12 \| \hat \nabla g(A_t)\|^2_\rf \quad \Longleftrightarrow \quad     \eps_{t+1} \leq \eps_t -  \frac12 \| \hat \nabla g(A_t)\|^2_\rf,
    \end{equation}
    where the equivalence is obtained by subtracting $g(A_\star) = f(S_\star)$ from both sides and defining $\eps_{t} := g(A_t) - g(A_\star)$. 
    Since $f$ is $\mu$-SC over $\abi$, we invoke the PL inequality (\cref{Lem:PLIneqfromSC}, with $[a I, b I] = \abi$, $\mu_0 = \mu$ and $a = \alpha$)
    \begin{equation}
        4 \mu \alpha \eps_t \leq \frac12 \| \hat \nabla g(A_t)\|_\rf^2.  
    \end{equation}
    Combining the above two expressions, and using the relation $4 \mu \alpha = \Lt(\frac{\alpha}{\beta}\Rt)^{3/2} = \kappa^{-3/2}$, yields $\eps_{t+1} \leq (1 - \kappa^{-3/2}) \eps_t$.
    Unrolling the inequality over $t$ iterations we get $\eps_{t} \leq (1 - \kappa^{-3/2})^t \eps_0$. 
    Substituting $\varepsilon_t = f(S_t) - f(S_\star)$ and $\varepsilon_0 = f(S_0) - f(S_\star)$, we get the required result, which is manifestly dimension-independent:
    \begin{equation}
        f(S_t) - f(S_\star) \leq \Lt(1 - \frac{1}{\kappa^{\frac32}}\Rt)^t (f(S_0) - f(S_\star)).
    \end{equation}
    This concludes the proof.
\end{proof}

\begin{remark}[Sharper constants on the refined interval] \label{Rem:RefinedConstants}
    For the barycenter problem the constants can be tightened: projecting onto $[\alpha' I, \beta' I]$---the refined bounds of~\cref{Eq:RefinedConstantsIntro}, which also enclose the optimum~\cite{altschuler2021averaging,bhatia2019bures}---yields the verbatim analogue of~\cref{Thm:UnifiedConvergence} in the refined condition number, improving the rate to $(1 - \kappa'^{-3/2})$ with $\kappa' := \beta'/\alpha' \leq \kappa$ and the iteration count of~\cref{Eq:IterationCount} to $\kappa'^{3/2} \log(1/\epsilon)$.
\end{remark}

The proof of convergence with projection onto the refined interval does not carry over automatically for the following reasons. 
The projection step of~\cref{Prop:ProjectedDescent} relied on each $R_i$ being contained in $\abi$, which fails on the refined interval $\abid$. 
What is needed instead is that clipping to the refined interval does not increase $f_1$, which we prove as~\cref{Prop:RefinedMonotone}. 
We nonetheless state our results on $\abi$, whose constants are read off the individual inputs and match the conventions of the literature (the two coincide in the worst case); since our $\kappa'$ is no larger than the refined condition number of~\cite{altschuler2021averaging}, the improvement in~\cref{tab:comparison} is meaningful instance-wise: $\kappa'^{3/2}$ against $O(\kappa'^{5/2})$.
The analysis of the refined bounds is deferred to~\cref{App:RefinedInterval}.
\subsection{Choosing the iteration count} \label{Sec:Certificate}
From $f(S_t) - f(S_\star) \leq (1 - \kappa^{-3/2})^t (f(S_0) - f(S_\star)) \leq e^{-t \kappa^{-3/2}} (f(S_0) - f(S_\star))$, it suffices to take
\begin{equation} \label{Eq:IterationCount}
    T  \geq \kappa^{3/2} \log({1/\epsilon})
\end{equation}
to reach relative $\epsilon$ accuracy in function value, $f(S_T) - f(S_\star) \leq \epsilon \, (f(S_0) - f(S_\star))$, matching the convention of~\cref{tab:comparison}.
This a priori count is conservative in practice; the following a posteriori alternative is sharper.

\begin{tbox}
    \begin{cor}[A posteriori certificate] \label{Cor:Certificate}
        Let $f \in \{f_1, f_2\}$ and let $(S_t)_{t \geq 0}$ be generated by~\cref{alg:projected_bwgd} from $S_0 \in \abi$ (additionally $S_0 \in \clp$ when $f = f_2$).
        For every $t \geq 0$,
        \begin{equation} \label{Eq:Certificate}
            f(S_t) - f(S_\star) \leq \frac{\kappa^{3/2}}{2} \, \mathrm{B}\Lt(S_t, \rmk(S_t)\Rt).
        \end{equation}
        Consequently, for an absolute tolerance $\delta > 0$ one may terminate once $\mathrm{B}(S_t, \rmk(S_t)) \leq 2 \kappa^{-3/2} \delta$, at which point $f(S_t) - f(S_\star) \leq \delta$; taking $\delta = \epsilon\,(f(S_0) - f(S_\star))$ recovers the relative accuracy of~\cref{Eq:IterationCount}.
    \end{cor}
\end{tbox}
\begin{proof}
    Substitute the step-length identity $\mathrm{B}(S_t, \rmk(S_t)) = \|\hat \nabla g(A_t)\|_\rf^2$ of~\cref{Prop:ProjectedDescent} into the PL inequality of~\cref{Lem:PLIneqfromSC} and use $4\mu\alpha = \kappa^{-3/2}$.
\end{proof}

The certificate quantity is the (squared) length of the step the iteration has just computed; in practice the resulting count is far smaller than~\cref{Eq:IterationCount}.\footnote{Across the experiments of~\cref{Sec:Numerics}, the certificate (\cref{Cor:Certificate}) overestimates the true gap $(f(S_t) - f(S_\star))$ by roughly its $\kappa^{3/2}/2$ prefactor.
However, since the convergence is linear, this multiplicative slack costs only a few additional iterations, and the certified stopping time improves on the a priori count of~\cref{Eq:IterationCount} by orders of magnitude.}

\section{Numerical Experiments} \label{Sec:Numerics}
\begin{figure}[h]
    \centering
    \includegraphics[width=\textwidth]{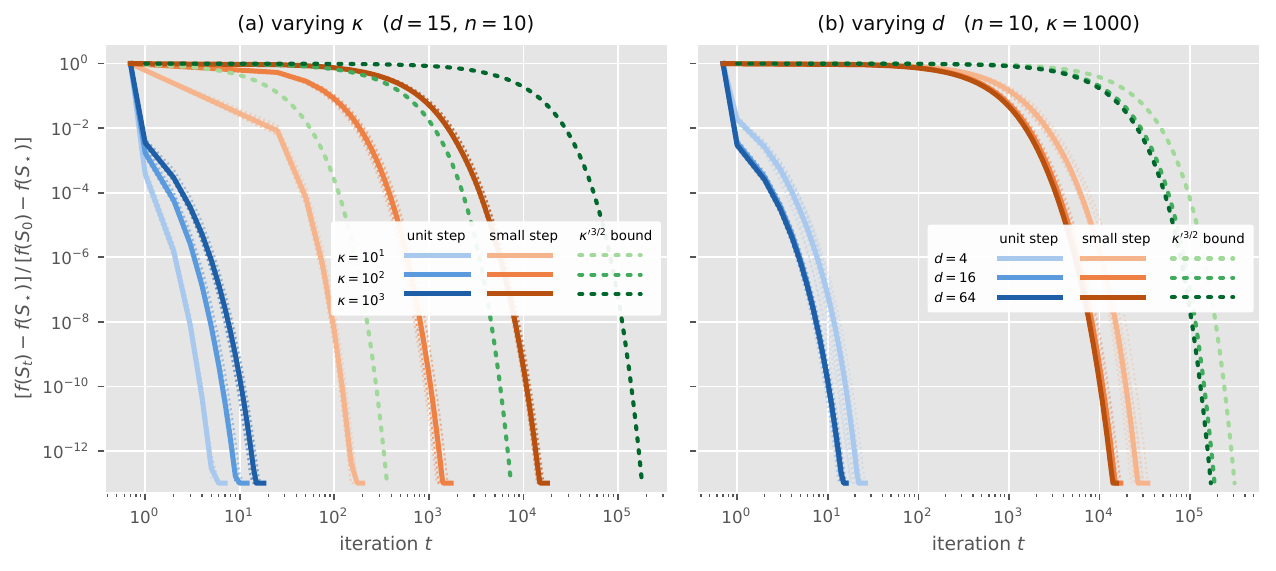}
    \caption{Unit-step (blue) against small-step (orange) RGD on \emph{pinned} ensembles (every $R_i$ of condition number exactly $\kappa$), $12$ instances per band (translucent) with opaque medians; darker shades for larger parameter values, gaps normalized by their initial value.
    (a) Varying $\kappa$ at $d = 15$: median iterations to the numerical floor range over $6$ to $16$ for the unit step, against $175$ to $16288$ for the small step.
    (b) Varying $d$ at $\kappa = 10^3$: the unit-step counts are essentially flat over a sixteenfold increase in $d$, as~\cref{Thm:UnifiedConvergence} asserts; the small-step counts track the conditioning $\kappa'$, not the dimension (see text).
    Dotted curves show the a priori guarantee of~\cref{Cor:RefinedRate} at each setting's median $\kappa'$; the clip never activates, so the projected and unprojected unit steps coincide.
    Note the logarithmic iteration axis.
    }
    \label{Fig:UnitVsSmall}
\end{figure}

We now discuss the numerical experiments accompanying our theoretical results. 
The scripts and notebook that generate every figure and number reported here are available at the accompanying repository~\cite{afham2026projectedrgdcode}.
First, we compare our Projected RGD algorithm against the standard (unprojected) RGD algorithm and the \textit{small-step} RGD algorithm of~\cite{altschuler2021averaging}, run at the step size $\eta = \alpha'/(2\beta')$, no smaller than its prescription $\alpha'/(2\sumin w_i \lmax(R_i))$ and hence conservative in its favor.
On the instance of~\cref{Ex:FeasibleExit} (shown in \cref{Fig:Headline}(b)), we ran projected against unprojected RGD: the projection is active on exactly the two exit steps, and the two trajectories are indistinguishable in convergence rate thereafter.
On the same instance the small step, $\eta = \alpha'/(2\beta') = 1.33 \times 10^{-2}$, remains feasible throughout but needs $544$ iterations to reach a $10^{-6}$ relative gap and $1108$ to reach $10^{-12}$, over a window in which the unit step converges to machine precision.
Unit-step RGD reaches the same relative gaps of $10^{-6}$ and $10^{-12}$ at iterations $3$ and $6$ respectively (projected and unprojected alike), a factor of about $180$ fewer iterations.
Numerics also suggest that unless the ensemble and the initial iterate are adversarially designed, all iterates remain in $\abi$: in every subsequent (non-adversarial) experiment the clipping never activates, so the projected and unprojected unit-step RGDs coincide, and the projection acts as a dormant safeguard that converts the numerically observed dimension-independent linear convergence into a guarantee.

The primary comparison, in terms of performance, is then between the (projected) unit-step and the small-step RGD.
We demonstrate this in \cref{Fig:UnitVsSmall}(a), which shows RGD exhibits a far superior performance at unit step than at small step: at $\kappa = 10^3$ the unit step reaches the numerical floor in $16$ iterations against 16288 for the small step, three orders of magnitude apart.
This empirical superiority has been observed previously~\cite{altschuler2021averaging,brahmachari2025fixed}, but a theoretical backing was lacking.
Our key contribution is to ground this empirical performance in theory, at the cost of a projection that is (i) essentially free and (ii) almost never needed in practice.
We then empirically study the rate of convergence for RGD at unit- and small-step sizes.
As seen in \cref{Fig:UnitVsSmall}(a), the observed contraction per iteration is far faster than the guaranteed rate $(1 - \kappa^{-3/2})$---and remains so under the refined rate $(1 - \kappa'^{-3/2})$ of~\cref{Rem:RefinedConstants}---so the iteration counts fall well short of the a priori count $\kappa^{3/2}\log(1/\epsilon)$ of~\cref{Eq:IterationCount}, suggesting room for improvement in the constants of~\cref{Thm:UnifiedConvergence}.

Fixing $\kappa$ and grading instead by the dimension, \cref{Fig:UnitVsSmall}(b) isolates the dimension-independence of the guarantee.
For the sweep to be meaningful the effective difficulty must not itself vary with $d$: with i.i.d.\ spectra, $\min_i \lmin(R_i)$ is a minimum of $d$ draws and collapses onto $\alpha$ as $d$ grows, so the refined constants---and with them the small-step size and the sharpened bound---would drift for reasons unrelated to the rate; the pinned ensembles of~\cref{Fig:UnitVsSmall} remove this drift, leaving only a mild, saturating decrease of $\beta'$ as the eigenbases decorrelate.
The unit-step bands for $d \in \{4, 16, 64\}$ then nearly coincide, as do the small-step bands and the sharpened bounds, the residual spread of all three tracking $\kappa'$ rather than $d$: the iteration count is governed by the conditioning, not the dimension, as~\cref{Thm:UnifiedConvergence} asserts.

\section{Conclusion}
We introduced and analyzed Projected BW-GD: unit-step RGD composed, at every iteration, with the BW projection onto the well-conditioned set, which we showed is eigenvalue clipping (\cref{Lem:ProjectionLemma}).
For both the BW barycenter and the invariant matrix projection problem, the algorithm converges linearly at the dimension-independent rate $(1 - \kappa^{-3/2})$---the first such guarantee at the unit step size used in practice, and a polynomial improvement, from $O(\kappa'^{5/2})$ to $\kappa'^{3/2}$, in the iteration complexity of the best dimension-free guarantee---with an a posteriori certificate permitting termination on the fly.
We do not establish the rate for the unprojected fixed-point algorithm; rather, an eigenvalue clip of negligible cost converts it into an algorithm with a worst-case guarantee, while coinciding with it on typical inputs.
The analysis runs on the flat square-root manifold, where Euclidean strong convexity becomes a PL inequality and the unit-step descent lemma is a variance identity.
This perspective greatly simplifies the technical analysis by avoiding notions such as convexity along generalized geodesics. 
Both ingredients require well-conditioned iterates---a requirement that \cref{Ex:FeasibleExit} shows can genuinely fail for the unprojected iteration---and the projection enforces it without undoing the progress of the gradient step.
The invariant matrix projection problem inherits this picture in full, since the invariant submanifold is totally geodesic and closed under the clip (\cref{App:BRT_Algorithm}); the rate thereby extends to the quantum-information problems that the formulation unifies~\cite[Sec.~7]{brahmachari2025fixed}.

Two directions are suggested by the experiments.
First, from $S_0 = \frac{\alpha + \beta}{2} I$ the unprojected trajectory remained in $\abi$ in all our numerics, and the exits of~\cref{Ex:FeasibleExit} required an adversarial initial point; proving that unprojected unit-step RGD preserves $\abi$ from this canonical initialization would render the projection unnecessary for a default $S_0$.
Second, the observed rates are far faster than the bound, even after the sharpening of~\cref{App:RefinedInterval}: the improvement from $\kappa^{5/2}$ to $\kappa^{3/2}$ is an improvement of the \emph{guarantee}, and both guarantees are loose against what the iteration actually does.
Whether the generic-case rate admits a sharper exponent, or the worst case is attained only on rare instances, remains open.

\section*{Acknowledgments}
The author thanks Marco Tomamichel and Roberto Rubboli for helpful discussions. 
This project is supported by the NRF Investigatorship award (NRF-NRFI10-2024-0006) and the CQT Young Researcher Career Development Grant. 

\section*{Use of Large Language Models}
The author acknowledges the use of Large Language Models (LLMs) in the preparation of this article.
This includes \textit{standard} usage such as proofreading, restructuring, basic questions, and coding.
In terms of \textit{meaningful} contribution, the latter half of~\cref{Lem:ProjectionLemma} was proven using Google Gemini 3.1 Pro.
The key prompt, paraphrased, was \textit{``Show that (two-sided) clipping of eigenvalues to an interval is a contraction with respect to the BW distance''}, which was inspired by the existing result~\cite[Prop. 3]{altschuler2021averaging}, which showed that clipping eigenvalues from above is a BW contraction. 
As mentioned in the main text of the article, this is important as existing proof methods were unable to guarantee dimension-independent convergence (at unit step-size) because the iterates could dip below $\alpha I$, which can now be remedied by clipping since it acts as a contraction.
The LLM-generated proof was then reworked into the current form for improved presentation. 
Given the proof for contraction, it was then straightforward to show that eigenvalue-clipping constituted a BW-projection onto $\abi$ (former half of~\cref{Lem:ProjectionLemma}). 
Other technical results obtained with the help of LLMs (primarily Claude Fable) include~\cref{Prop:RefinedMonotone}, which proved that clipping onto the refined interval $\abid$ does not increase the objective value,~\cref{Cor:FrobeniusProjection}, which showed that singular-value clipping corresponds to Frobenius projection onto the set of square-roots of the well-conditioned interval, and the counterexample instance of~\cref{Ex:FeasibleExit}.\footnote{See~\cref{App:FeasibleExit} for the search strategy.}

\printbibliography

@THESIS{Afham2025Thesis,
  AUTHOR = {Afham, A},
  URL = {https://opus.lib.uts.edu.au/handle/10453/190570},
  DATE = {2025},
  LANGID = {american},
  TITLE = {Optimizing and Generalizing Quantum Fidelities},
  TYPE = {Thesis},
}

@MISC{afham2026projectedrgdcode,
  AUTHOR = {Afham, A.},
  URL = {https://github.com/afhamash/projected-rgd},
  DATE = {2026},
  TITLE = {Companion code for ``Projected {R}iemannian Gradient Descent for the {B}ures--{W}asserstein Barycenter''},
}

@ARTICLE{Afham2025Riemanniangeometric,
  AUTHOR = {Afham, A. and Ferrie, Chris},
  PUBLISHER = {AIP Publishing},
  DATE = {2025},
  ISBN = {0022-2488},
  JOURNALTITLE = {Journal of Mathematical Physics},
  NUMBER = {8},
  TITLE = {Riemannian-Geometric Generalizations of Quantum Fidelities and {{Bures-Wasserstein}} Distance},
  VOLUME = {66},
}

@MISC{Afham2022Quantum,
  AUTHOR = {Afham, A. and Kueng, Richard and Ferrie, Chris},
  URL = {https://arxiv.org/abs/2206.08183},
  DATE = {2022-06},
  EPRINT = {2206.08183},
  EPRINTCLASS = {quant-ph},
  EPRINTTYPE = {arXiv},
  LANGID = {english},
  SHORTTITLE = {Quantum Mean States Are Nicer than You Think},
  TITLE = {Quantum Mean States Are Nicer than You Think: Fast Algorithms to Compute States Maximizing Average Fidelity},
}

@MISC{Afham2026Projections,
  AUTHOR = {Afham, A. and Tomamichel, Marco},
  URL = {https://arxiv.org/abs/2602.14732},
  DATE = {2026},
  EPRINT = {2602.14732},
  EPRINTCLASS = {quant-ph},
  EPRINTTYPE = {arXiv},
  TITLE = {Projections with Respect to Bures Distance and Fidelity: Closed-Forms and Applications},
}

@ARTICLE{agueh2011barycenters,
  AUTHOR = {Agueh, Martial and Carlier, Guillaume},
  PUBLISHER = {SIAM},
  DATE = {2011},
  JOURNALTITLE = {SIAM Journal on Mathematical Analysis},
  NUMBER = {2},
  PAGES = {904--924},
  TITLE = {Barycenters in the Wasserstein Space},
  VOLUME = {43},
}

@ARTICLE{aksenov2026anderson,
  AUTHOR = {Aksenov, Vitalii and Eigel, Martin and Oster, Mathias},
  DATE = {2026},
  JOURNALTITLE = {arXiv preprint arXiv:2601.22038},
  TITLE = {Anderson Mixing in Bures Wasserstein Space of Gaussian Measures},
}

@ARTICLE{alberti1983note,
  AUTHOR = {Alberti, Peter M},
  PUBLISHER = {Springer},
  DATE = {1983},
  JOURNALTITLE = {Letters in Mathematical Physics},
  NUMBER = {1},
  PAGES = {25--32},
  TITLE = {A note on the transition probability over C*-algebras},
  VOLUME = {7},
}

@ARTICLE{altschuler2021averaging,
  AUTHOR = {Altschuler, Jason and Chewi, Sinho and Gerber, Patrik R and Stromme, Austin},
  DATE = {2021},
  JOURNALTITLE = {Advances in Neural Information Processing Systems},
  PAGES = {22132--22145},
  TITLE = {Averaging on the {{Bures-Wasserstein}} Manifold: Dimension-Free Convergence of Gradient Descent},
  VOLUME = {34},
}

@ARTICLE{alvarez2016fixed,
  AUTHOR = {Álvarez-Esteban, Pedro C and Del Barrio, E and {Cuesta-Albertos}, {\relax JA} and Matrán, C},
  PUBLISHER = {Elsevier},
  DATE = {2016},
  JOURNALTITLE = {Journal of Mathematical Analysis and Applications},
  NUMBER = {2},
  PAGES = {744--762},
  TITLE = {A Fixed-Point Approach to Barycenters in {{Wasserstein}} Space},
  VOLUME = {441},
}

@BOOK{ambrosio2008gradient,
  AUTHOR = {Ambrosio, Luigi and Gigli, Nicola and Savaré, Giuseppe},
  PUBLISHER = {Springer Science \& Business Media},
  DATE = {2008},
  TITLE = {Gradient Flows: In Metric Spaces and in the Space of Probability Measures},
}

@ARTICLE{beigi2013sandwiched,
  AUTHOR = {Beigi, Salman},
  PUBLISHER = {AIP Publishing},
  DATE = {2013},
  JOURNALTITLE = {Journal of Mathematical Physics},
  NUMBER = {12},
  TITLE = {Sandwiched Rényi divergence satisfies data processing inequality},
  VOLUME = {54},
}

@INCOLLECTION{BhatiaPD,
  AUTHOR = {Bhatia, Rajendra},
  PUBLISHER = {Princeton university press},
  BOOKTITLE = {Positive Definite Matrices},
  DATE = {2009},
  DOI = {10.1515/9781400827787},
  TITLE = {Positive Definite Matrices},
}

@ARTICLE{bhatia2019bures,
  AUTHOR = {Bhatia, Rajendra and Jain, Tanvi and Lim, Yongdo},
  PUBLISHER = {Elsevier},
  DATE = {2019},
  JOURNALTITLE = {Expositiones Mathematicae},
  NUMBER = {2},
  PAGES = {165--191},
  TITLE = {On the {{Bures}}--{{Wasserstein}} Distance between Positive Definite Matrices},
  VOLUME = {37},
}

@ARTICLE{bhatia2018strong,
  AUTHOR = {Bhatia, Rajendra and Jain, Tanvi and Lim, Yongdo},
  PUBLISHER = {World Scientific},
  DATE = {2018},
  JOURNALTITLE = {Reviews in Mathematical Physics},
  NUMBER = {09},
  PAGES = {1850014},
  TITLE = {Strong Convexity of Sandwiched Entropies and Related Optimization Problems},
  VOLUME = {30},
}

@ARTICLE{brahmachari2025fixed,
  AUTHOR = {Brahmachari, Shrigyan and Rubboli, Roberto and Tomamichel, Marco},
  PUBLISHER = {Springer},
  DATE = {2025},
  JOURNALTITLE = {Mathematical Programming},
  PAGES = {1--31},
  TITLE = {A fixed-point algorithm for matrix projections with applications in quantum information: S. Brahmachari et al.},
}

@ARTICLE{burer2003nonlinear,
  AUTHOR = {Burer, Samuel and Monteiro, Renato DC},
  PUBLISHER = {Springer},
  DATE = {2003},
  JOURNALTITLE = {Mathematical programming},
  NUMBER = {2},
  PAGES = {329--357},
  TITLE = {A nonlinear programming algorithm for solving semidefinite programs via low-rank factorization},
  VOLUME = {95},
}

@ARTICLE{burer2005local,
  AUTHOR = {Burer, Samuel and Monteiro, Renato DC},
  PUBLISHER = {Springer},
  DATE = {2005},
  JOURNALTITLE = {Mathematical programming},
  NUMBER = {3},
  PAGES = {427--444},
  TITLE = {Local minima and convergence in low-rank semidefinite programming},
  VOLUME = {103},
}

@BOOK{chewi2025statistical,
  AUTHOR = {Chewi, Sinho and Niles-Weed, Jonathan and Rigollet, Philippe},
  PUBLISHER = {Springer},
  DATE = {2025},
  TITLE = {Statistical optimal transport},
}

@INPROCEEDINGS{chewi2020gradient,
  AUTHOR = {Chewi, Sinho and Maunu, Tyler and Rigollet, Philippe and Stromme, Austin J},
  PUBLISHER = {PMLR},
  BOOKTITLE = {Conference on Learning Theory},
  DATE = {2020},
  PAGES = {1276--1304},
  TITLE = {Gradient Descent Algorithms for {{Bures-Wasserstein}} Barycenters},
}

@ARTICLE{gupta2015multiplicativity,
  AUTHOR = {Gupta, Manish K and Wilde, Mark M},
  PUBLISHER = {Springer},
  DATE = {2015},
  JOURNALTITLE = {Communications in Mathematical Physics},
  NUMBER = {2},
  PAGES = {867--887},
  TITLE = {Multiplicativity of completely bounded p-norms implies a strong converse for entanglement-assisted capacity},
  VOLUME = {334},
}

@ARTICLE{jozsa1994fidelity,
  AUTHOR = {Jozsa, Richard},
  PUBLISHER = {Taylor \& Francis},
  DATE = {1994},
  JOURNALTITLE = {Journal of modern optics},
  NUMBER = {12},
  PAGES = {2315--2323},
  TITLE = {Fidelity for Mixed Quantum States},
  VOLUME = {41},
}

@InProceedings{junyi2024convergence,
  title = 	 {On the Convergence of Projected Bures-{W}asserstein Gradient Descent under {E}uclidean Strong Convexity},
  author =       {Fan, Junyi and Han, Yuxuan and Liu, Zijian and Cai, Jian-Feng and Wang, Yang and Zhou, Zhengyuan},
  booktitle = 	 {Proceedings of the 41st International Conference on Machine Learning},
  pages = 	 {12832--12857},
  year = 	 {2024},
  volume = 	 {235},
  series = 	 {Proceedings of Machine Learning Research},
  month = 	 {21--27 Jul},
  publisher =    {PMLR},
  url = 	 {https://proceedings.mlr.press/v235/fan24b.html},
}

@INPROCEEDINGS{karimi2016linear,
  AUTHOR = {Karimi, Hamed and Nutini, Julie and Schmidt, Mark},
  ORGANIZATION = {Springer},
  BOOKTITLE = {Joint European conference on machine learning and knowledge discovery in databases},
  DATE = {2016},
  PAGES = {795--811},
  TITLE = {Linear convergence of gradient and proximal-gradient methods under the polyak-{ł}ojasiewicz condition},
}

@ARTICLE{Konig2009,
  AUTHOR = {Konig, Robert and Renner, Renato and Schaffner, Christian},
  PUBLISHER = {Institute of Electrical and Electronics Engineers (IEEE)},
  URL = {https://doi.org/10.1109\%2Ftit.2009.2025545},
  DATE = {2009},
  DOI = {10.1109/tit.2009.2025545},
  JOURNALTITLE = {IEEE Transactions on Information Theory},
  NUMBER = {9},
  PAGES = {4337--4347},
  TITLE = {The Operational Meaning of Min- and Max-Entropy},
  VOLUME = {55},
}

@ARTICLE{kroshnin2021statistical,
  AUTHOR = {Kroshnin, Alexey and Spokoiny, Vladimir and Suvorikova, Alexandra},
  PUBLISHER = {Institute of Mathematical Statistics},
  DATE = {2021},
  JOURNALTITLE = {The Annals of Applied Probability},
  NUMBER = {3},
  PAGES = {1264--1298},
  TITLE = {Statistical Inference for {{Bures}}--{{Wasserstein}} Barycenters},
  VOLUME = {31},
}

@ARTICLE{lambert2022variational,
  AUTHOR = {Lambert, Marc and Chewi, Sinho and Bach, Francis and Bonnabel, Silvère and Rigollet, Philippe},
  DATE = {2022},
  JOURNALTITLE = {Advances in Neural Information Processing Systems},
  PAGES = {14434--14447},
  TITLE = {Variational inference via Wasserstein gradient flows},
  VOLUME = {35},
}

@ARTICLE{le2017existence,
  AUTHOR = {Le Gouic, Thibaut and Loubes, Jean-Michel},
  PUBLISHER = {Springer},
  DATE = {2017},
  JOURNALTITLE = {Probability Theory and Related Fields},
  NUMBER = {3},
  PAGES = {901--917},
  TITLE = {Existence and consistency of Wasserstein barycenters},
  VOLUME = {168},
}

@ARTICLE{le2022fast,
  AUTHOR = {Le Gouic, Thibaut and Paris, Quentin and Rigollet, Philippe and Stromme, Austin J},
  DATE = {2022},
  JOURNALTITLE = {Journal of the European Mathematical Society},
  NUMBER = {6},
  PAGES = {2229--2250},
  TITLE = {Fast convergence of empirical barycenters in Alexandrov spaces and the Wasserstein space},
  VOLUME = {25},
}

@ARTICLE{lewis2008alternating,
  AUTHOR = {Lewis, Adrian S and Malick, Jérôme},
  PUBLISHER = {INFORMS},
  DATE = {2008},
  JOURNALTITLE = {Mathematics of Operations Research},
  NUMBER = {1},
  PAGES = {216--234},
  TITLE = {Alternating projections on manifolds},
  VOLUME = {33},
}

@ARTICLE{liu2017new,
  AUTHOR = {Liu, CL and Zhang, Da-Jian and Yu, Xiao-Dong and Ding, Qi-Ming and Liu, Longjiang},
  PUBLISHER = {Springer},
  DATE = {2017},
  JOURNALTITLE = {Quantum Information Processing},
  NUMBER = {8},
  PAGES = {198},
  TITLE = {A new coherence measure based on fidelity},
  VOLUME = {16},
}

@ARTICLE{malago2018wasserstein,
  AUTHOR = {Malago, Luigi and Montrucchio, Luigi and Pistone, Giovanni},
  DATE = {2018},
  EPRINT = {1801.09269},
  EPRINTTYPE = {arXiv},
  JOURNALTITLE = {arXiv preprint arXiv:1801.09269},
  TITLE = {Wasserstein {{Riemannian}} Geometry of Positive Definite Matrices},
}

@INPROCEEDINGS{maunu2023bures,
  AUTHOR = {Maunu, Tyler and Le Gouic, Thibaut and Rigollet, Philippe},
  ORGANIZATION = {PMLR},
  BOOKTITLE = {International Conference on Artificial Intelligence and Statistics},
  DATE = {2023},
  PAGES = {8183--8210},
  TITLE = {Bures-wasserstein barycenters and low-rank matrix recovery},
}

@ARTICLE{oneill1966fundamental,
  AUTHOR = {O'Neill, Barrett},
  PUBLISHER = {University of Michigan, Department of Mathematics},
  DATE = {1966},
  JOURNALTITLE = {Michigan Mathematical Journal},
  NUMBER = {4},
  PAGES = {459--469},
  TITLE = {The fundamental equations of a submersion.},
  VOLUME = {13},
}

@ARTICLE{panaretos2019statistical,
  AUTHOR = {Panaretos, Victor M and Zemel, Yoav},
  PUBLISHER = {Annual Reviews},
  DATE = {2019},
  JOURNALTITLE = {Annual review of statistics and its application},
  NUMBER = {1},
  PAGES = {405--431},
  TITLE = {Statistical Aspects of {{Wasserstein}} Distances},
  VOLUME = {6},
}

@ARTICLE{polyak1963gradient,
  AUTHOR = {Polyak, Boris Teodorovich},
  PUBLISHER = {Russian Academy of Sciences, Branch of Mathematical Sciences},
  DATE = {1963},
  JOURNALTITLE = {Zhurnal vychislitel'noi matematiki i matematicheskoi fiziki},
  NUMBER = {4},
  PAGES = {643--653},
  TITLE = {Gradient methods for minimizing functionals},
  VOLUME = {3},
}

@MISC{tomamichel2013thesis,
  AUTHOR = {Tomamichel, Marco},
  URL = {https://arxiv.org/abs/1203.2142},
  DATE = {2013},
  EPRINT = {1203.2142},
  EPRINTCLASS = {quant-ph},
  EPRINTTYPE = {arXiv},
  TITLE = {A Framework for Non-Asymptotic Quantum Information Theory},
}

@BOOK{tomamichel2015quantum,
  AUTHOR = {Tomamichel, Marco},
  PUBLISHER = {Springer},
  DATE = {2015},
  TITLE = {Quantum Information Processing with Finite Resources: Mathematical Foundations},
  VOLUME = {5},
}

@ARTICLE{uhlmann1976transition,
  AUTHOR = {Uhlmann, Armin},
  PUBLISHER = {Elsevier},
  DATE = {1976},
  JOURNALTITLE = {Reports on Mathematical Physics},
  NUMBER = {2},
  PAGES = {273--279},
  TITLE = {The ``Transition Probability'' in the State Space of a Star-Algebra},
  VOLUME = {9},
}

@ARTICLE{Uhlmann2011Transition,
  AUTHOR = {Uhlmann, Armin},
  PUBLISHER = {Springer},
  DATE = {2011},
  ISBN = {0015-9018},
  JOURNALTITLE = {Foundations of physics},
  NUMBER = {3},
  PAGES = {288--298},
  TITLE = {Transition Probability (Fidelity) and Its Relatives},
  VOLUME = {41},
}

@BOOK{Watrous2018Theory,
  AUTHOR = {Watrous, John},
  LOCATION = {Cambridge},
  PUBLISHER = {Cambridge University Press},
  DATE = {2018},
  DOI = {10.1017/9781316848142},
  ISBN = {978-1-107-18056-7},
  TITLE = {The {{Theory}} of {{Quantum Information}}},
}

@ARTICLE{wei2003geometric,
  AUTHOR = {Wei, Tzu-Chieh and Goldbart, Paul M},
  PUBLISHER = {APS},
  DATE = {2003},
  JOURNALTITLE = {Physical Review A},
  NUMBER = {4},
  PAGES = {042307},
  TITLE = {Geometric measure of entanglement and applications to bipartite and multipartite quantum states},
  VOLUME = {68},
}

@ARTICLE{winter2016operational,
  AUTHOR = {Winter, Andreas and Yang, Dong},
  PUBLISHER = {APS},
  DATE = {2016},
  JOURNALTITLE = {Physical review letters},
  NUMBER = {12},
  PAGES = {120404},
  TITLE = {Operational resource theory of coherence},
  VOLUME = {116},
}

@ARTICLE{zemel2019frechet,
  AUTHOR = {Zemel, Yoav and Panaretos, Victor M.},
  PUBLISHER = {Bernoulli Society for Mathematical Statistics and Probability},
  URL = {https://doi.org/10.3150/17-BEJ1009},
  DATE = {2019},
  DOI = {10.3150/17-BEJ1009},
  JOURNALTITLE = {Bernoulli. Official Journal of the Bernoulli Society for Mathematical Statistics and Probability},
  NUMBER = {2},
  PAGES = {932--976},
  TITLE = {Fréchet Means and {{Procrustes}} Analysis in {{Wasserstein}} Space},
  VOLUME = {25},
}

@ARTICLE{zheng2025riemannian,
  AUTHOR = {Zheng, Shixin and Huang, Wen and Vandereycken, Bart and Zhang, Xiangxiong},
  PUBLISHER = {Springer},
  DATE = {2025},
  JOURNALTITLE = {Computational Optimization and Applications},
  NUMBER = {3},
  PAGES = {1135--1184},
  TITLE = {Riemannian optimization using three different metrics for Hermitian PSD fixed-rank constraints: S. Zheng et al.},
  VOLUME = {91},
}

\appendix
\crefalias{section}{appendix}

\section{Unit-Step RGD Exits the Well-Conditioned Set} \label{App:FeasibleExit}
We now discuss a barycenter instance, where all the matrices involved---the elements of the ensemble, the initial point, and the barycenter---lie strictly in the interior of a compact interval $\abi$, yet the trajectory exits the interval (transiently). 
\cref{Fig:FeasibleExit} displays the trajectory and the formal statement follows. 
This is the same instance, and the same trajectory, as panel (b) of the headline~\cref{Fig:Headline}; here it is shown in the complementary view---the full two-eigenvalue picture in the main panel, the signed deviation in the inset---alongside the formal statement.

\begin{figure}[h]
    \centering
    \includegraphics[width=0.82\textwidth]{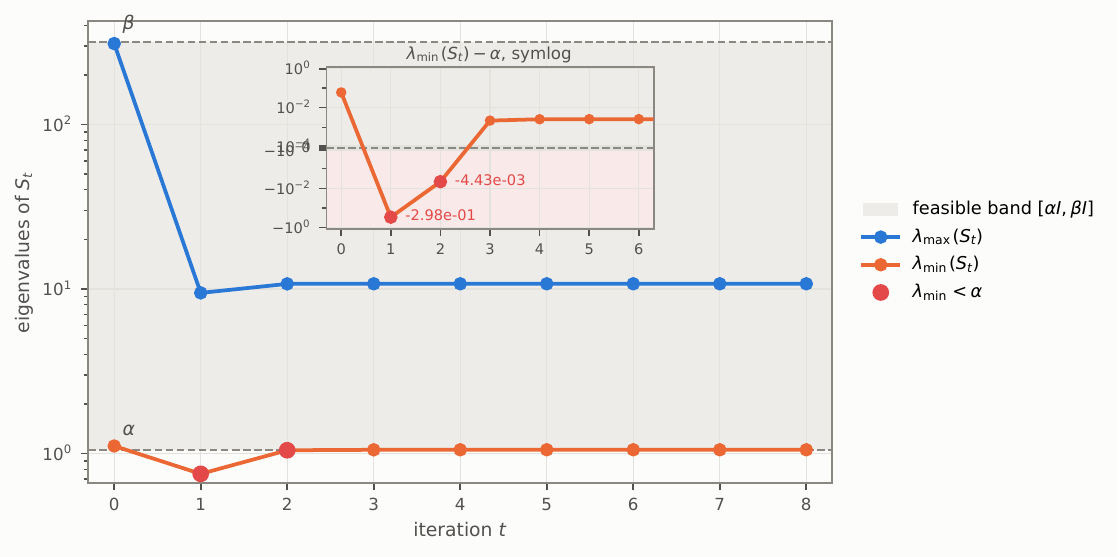}
    \caption{The trajectory of~\cref{Ex:FeasibleExit}.
    Main panel: both eigenvalues of the unprojected unit-step RGD iterates (log scale), with the feasible band $\abi$ shaded and the infeasible iterates marked.
    Inset: the signed deviation $\lambda_{\min}(S_t) - \alpha$ on a symmetric-log scale, resolving the two undershoots ($-2.98 \times 10^{-1}$ and $-4.43 \times 10^{-3}$) against the limit $\lambda_{\min}(S_\star) - \alpha \approx +2.52 \times 10^{-3}$: the iterate exits $\abi$ at $t = 1$, is still outside at $t = 2$, and re-enters at $t = 3$.}
    \label{Fig:FeasibleExit}
\end{figure}

\begin{example}[Transient exit of Unit-step RGD] \label{Ex:FeasibleExit}
    Let $d = 2$, $Q_\theta := \begin{pmatrix} \cos (\theta) &  \sin (\theta) \\ - \sin(\theta) & \cos (\theta) \end{pmatrix}$ be the rotation matrix by angle $\theta$, and consider the barycenter instance with weights $w = (0.12,\, 0.827,\, 0.053)$ and ensemble elements
    \begin{equation}
        R_1 = Q_{\theta_1} \begin{pmatrix}
            1.05 & 0 \\
            0 & 315 \\ 
        \end{pmatrix} Q_{\theta_1}^\top, \,\,
        R_2 = Q_{\theta_2} \begin{pmatrix}
            1.0501 & 0 \\
            0 & 1.575 \\ 
        \end{pmatrix} Q_{\theta_2}^\top, \,\,
        R_3 = Q_{\theta_3} \begin{pmatrix}
            1.0502 & 0 \\
            0 & 4.41 \\ 
        \end{pmatrix} Q_{\theta_3}^\top,
    \end{equation}
    where the angles are $(\theta_1, \theta_2, \theta_3) = (0.246, 0.338, 0.253)$.
    We choose $\alpha = 1.05$ and $\beta = 315$ ($\kappa = 300$), together with the strictly interior initialization
    \begin{equation}
        S_0 = Q_{\theta_0} \begin{pmatrix}
            1.11 & 0 \\
            0 & 310 \\ 
        \end{pmatrix} Q_{\theta_{0}}^\top \in \mathrm{interior}\,\abi;  \qquad \theta_0 = 1.473.
    \end{equation}
    Although the barycenter admits no closed form, its location is certified by the eigenvalue bounds of~\cite[Thm.~6]{altschuler2021averaging}:
    \begin{equation}
        \lambda_{\min}(S_\star) \geq  \Big(\sumin w_i \sqrt{\lmin(R_i)}\Big)^{2} =: \alpha',
        \quad
        \lambda_{\max}(S_\star) \leq \Lt(\sumin w_i \sqrt{\lmax(R_i)}\Rt)^{2} = 10.75,
    \end{equation}
    where $\alpha' = \alpha + 9.33 \times 10^{-5}$,
    so $S_\star$ lies \emph{strictly} in the interior of $\abi$.
    Running the unprojected unit-step RGD iteration from $S_0$ yields, in double precision,
    \begin{equation}
        \lambda_{\min}(S_t) - \alpha =
        \begin{cases}
            \, -2.984 \times 10^{-1} & t = 1,\\
            \, -4.434 \times 10^{-3} & t = 2,\\
            \, +2.117 \times 10^{-3} & t = 3.
        \end{cases}
    \end{equation}
    The trajectory therefore exits $\abi$ at $t = 1$ and \emph{remains outside for two consecutive iterations}---indeed, both excursion iterates lie below even the certified floor $\alpha'$ of the optimum, so the excursion cannot be attributed to the iterates approaching a near-boundary optimum---before re-entering at $t = 3$ and converging (numerically, $\lambda_{\min}(S_t) \to \alpha + 2.52 \times 10^{-3}$).
\end{example}

We choose $n = 3$ deliberately as for $n = 2$ the barycenter is a point on the geodesic joining the targets and admits a closed form.
The example extends to every problem size: to any $n \geq 3$ by splitting $R_3$ into identical copies whose weights sum to $w_3$, which leaves the iteration unchanged, and to any $d \geq 2$ by padding all matrices with a block $c\, I_{d-2}$, $c \in [\alpha, \beta]$, under which the dynamics act block-diagonally and the violations are unchanged.
The instance was found using an LLM-driven search (Claude Fable): a randomized sweep over small configurations ($d = 2$, $n = 3$) first located rare single-step exits, and an adversarial refinement of the angles, spectra, and weights then deepened these into the two-consecutive-step excursion above, while keeping the ensemble, the initial point, and the certified location of the optimum strictly interior.

\section{A Variance Inequality Proof of the Descent Lemma} \label{App:UsefulLemmas}
The descent lemma used in~\cref{Sec:Convergence} is~\cite[Theorem 10]{bhatia2019bures}, proved there with tools from optimal transport. 
We give a self-contained proof of the same result, phrased in the square-root manifold, where it reduces to the variance identity for a mean in a Euclidean space.
\begin{tbox}
    \begin{prop} \label{Prop:VarianceInequalityBuresDistance}
        Let $(R_i, w_i)_\inn$ be a weighted ensemble and $g_1 := f_1 \circ \pi$.
        For any $S \in \pdd$,
        \begin{equation}
            g_1(S \hf) \geq g_1(A') + \frac12 \|\nabla g_1(S \hf)\|_\mathrm{F}^2,
        \end{equation}
        where $A' = S \hf - \nabla g_1(S \hf)$; consequently
        \begin{equation}
            f_1(S) \geq f_1(\rmk(S)) + \frac12 \mathrm{B}(S, \rmk(S)).
        \end{equation}
    \end{prop}
\end{tbox}
\begin{proof}
    For any $(Y_i)_\inn \subseteq \mathrm{L}(d)$ with mean $\overline{Y} := \sumin w_i Y_i$ and any $X \in \mathrm{L}(d)$, expanding both sides verifies the variance identity
    \begin{equation} \label{Eq:VarianceIdentity}
        \sumin w_i \left\|X - Y_i\right\|_\mathrm{F}^2 = \sumin w_i \left\|\overline{Y} - Y_i\right\|_\mathrm{F}^2 + \left\|X - \overline{Y}\right\|_\mathrm{F}^2.
    \end{equation}
    Apply this with $X \equiv S \hf$ and $Y_i \equiv R_i \hf V_i^\dagger$, where $V_i = \mathrm{Pol}(S \hf R_i \hf)$, so that each $(S \hf , R_i \hf V_i^\dagger)$ is an optimal pairing: $\|S \hf - R_i \hf V_i^\dagger\|_\mathrm{F}^2 = \mathrm{B}(R_i, S)$.
    The left-hand side of~\cref{Eq:VarianceIdentity} is then $2 f_1(S) = 2 g_1(S\hf)$.
    The chain rule $\nabla g_1(S \hf) = 2 \nabla f_1(S) S \hf = \sumin w_i [S \hf - R_i \hf V_i^\dagger]$ (the second equality is~\cref{Prop:GradientPolarForm} applied termwise) can be rearranged as
    \begin{equation}
        \overline{Y} = \sumin w_i R_i \hf V_i^\dagger = S \hf - \nabla g_1(S \hf) = A',
    \end{equation}
    so $\|X - \overline{Y}\|_\mathrm{F}^2 = \|\nabla g_1(S \hf)\|_\mathrm{F}^2$.
    Finally, $\overline{Y} = A' \in \pi \iv [\rmk(S)]$ and $Y_i \in \pi \iv [R_i]$, so the variational characterization of the BW distance gives $\|\overline{Y} - Y_i\|_\mathrm{F}^2 \geq \mathrm{B}(R_i, \rmk(S))$, whence $\sumin w_i \|\overline{Y} - Y_i\|_\mathrm{F}^2 \geq 2 f_1(\rmk(S)) = 2 g_1(A')$.
    Substituting the three expressions into~\cref{Eq:VarianceIdentity} and dividing by two gives the first claim; the second follows since $\|\nabla g_1(S \hf)\|_\mathrm{F}^2 = \|S \hf - A'\|_\mathrm{F}^2 \geq \mathrm{B}(S, \rmk(S))$.
\end{proof}

We now see where the non-negative curvature of the manifold yields the descent lemma.  
The total manifold---equipped with the Frobenius metric---is flat, and there the statement is the Euclidean variance identity for the mean of the optimally paired square-roots. 
The inequality arises from descending to the base manifold, where the natural (BW) distance is obtained by \textit{minimizing} the Frobenius distance over the fibres (\cref{Eq:BuresMin}). 
Consequently, the quotient map contracts distances.
By O'Neill's formula for Riemannian submersions, sectional curvature can only increase under the quotient~\cite{oneill1966fundamental}, so the BW manifold inherits non-negative curvature from the flat total space; on Wasserstein space at large, the non-negative curvature holds in the Alexandrov sense~\cite[Thm.~7.3.2]{ambrosio2008gradient}.
In the proof, the contraction enters with the favorable sign: the averaged square roots project (down to the base manifold via $\pi$) to precisely the next RGD iterate, which is at least as close to each reference as the flat identity accounts for, so the gradient-norm term survives as a genuine per-step decrease.
This is an instance of a general phenomenon: non-negative curvature yields geodesic smoothness of the barycenter functional, and with it unit-step descent, on the whole Wasserstein space~\cite[Thm.~7]{chewi2020gradient}.

\section{Descent under Projection onto the Refined Interval} \label{App:RefinedInterval}
Throughout this appendix, $f_1$ is the barycenter functional and $\alpha' = \Lt[ \sumin w_i \sqrt{\lmin(R_i)} \Rt]^2$, $\beta':= \lmax \Lt( \sumin w_i R_i \Rt)$ are the \emph{refined} spectral bounds of~\cref{Eq:RefinedConstantsIntro}; they satisfy $\alpha \leq \alpha' \leq \beta' \leq \beta$, enclose the optimum ($\alpha' I \leq S_\star \leq \beta' I$; floor by~\cite[Thm.~6]{altschuler2021averaging}, ceiling by~\cite[Thm.~9]{bhatia2019bures}), and satisfy $\beta' \leq \sumin w_i \lmax(R_i)$, so $\kappa'$ is no larger than the refined condition number in which the small-step guarantee of~\cite[Thm.~3]{altschuler2021averaging} is stated.
Thus it is natural to ask if clipping onto the tighter interval $\abid$ will still lead to convergence of the Projected BW-GD algorithm, which does not follow from the arguments used in~\cref{Sec:Convergence}. 
This is because the descent argument of~\cref{Prop:ProjectedDescent} relied on the fact that each ensemble element $R_i$ lies in $\abi$ and is hence a \emph{fixed point} of $\clab$, which is what forced $f_1(\clab[\rmk(S)]) \leq f_1(\rmk(S))$.
The refined interval does not contain the ensemble---individual $R_i$ may have eigenvalues on either side of $[\alpha', \beta']$---so the fixed-point argument is unavailable.
By working with the objective function directly, we can show that clipping to this refined interval decreases the objective value, which allows us to use it in~\cref{alg:projected_bwgd}.
\begin{tbox}
    \begin{prop}[Refined clipping does not increase the objective] \label{Prop:RefinedMonotone}
        For every $S \in \pdd$, it holds that
        \begin{equation}
            f_1\Lt(\mathrm{clip}_{\alpha', \beta'}(S)\Rt) \leq f_1(S).
        \end{equation}
    \end{prop}
\end{tbox}
\begin{proof}
    Let $\hat S := \mathrm{clip}_{\alpha', \beta'}(S)$, and let $S = \sum_{j=1}^d \lambda_j v_j v_j \dg$ be the eigendecomposition; then $\hat S = \sum_j \mathrm{clip}_{\alpha', \beta'}(\lambda_j) v_j v_j\dg$ shares the eigenbasis, and
    \begin{equation}
        \Delta \equiv S - \hat S = \sum_{j=1}^d \delta_j \, v_j v_j \dg, \qquad \delta_j := \lambda_j - \mathrm{clip}_{\alpha', \beta'}(\lambda_j).
    \end{equation}
    The convexity of $f_1$ on $\pdd$ implies the following first-order convexity condition:
    \begin{equation}
        f_1(S) \geq f_1(\hat S) + \la \nabla f_1(\hat S), S - \hat S \ra.
    \end{equation}
    Thus it suffices to show $\la \nabla f_1(\hat S), S - \hat S \ra \geq 0$.
    By~\cref{Eq:GradBures}, $\nabla f_1(\hat S) = \frac12 (I - \overline{T})$ with $\overline{T} := \sumin w_i T_i$ and $T_i := \hat S \iv \# R_i$, where each $T_i$ is the unique positive definite solution of the Riccati equation $T_i \hat S T_i = R_i$~\cite{BhatiaPD}.
    Expanding in the shared eigenbasis,
    \begin{equation} \label{Eq:RefinedTermwise}
        \la \nabla f_1(\hat S), S - \hat S \ra = \frac12 \sum_{j=1}^d \delta_j (1 - \la v_j, \overline{T} v_j \ra).
    \end{equation}
    We will show that each summand is non-negative.
    Since $\lambda_j \in [\alpha', \beta'] \Rightarrow \delta_j = 0$, we can restrict our attention to the \textit{ceiling} case: $\lambda_j \geq \beta'$, and the \textit{floor} case: $\lambda_j \leq \alpha'$. 
    Let us first consider the ceiling case. 
    For $j$ such that $\delta_j = \lambda_j - \beta' > 0$, we have $\hat S v_j = \beta' v_j$.
    For each $i$, inserting $\hat S \hf \hat S \ihf$ and applying Cauchy--Schwarz,
    \begin{equation} \label{Eq:CeilingEstimate}
    \begin{aligned}
        \la v_j,  T_i v_j \ra  = \la \hat S \hf T_i v_j, \hat S \ihf v_j \ra \leq \Lt\|\hat S \hf T_i v_j\Rt\| \Lt\|\hat S \ihf v_j\Rt\| &= \sqrt{\la v_j, T_i \hat S T_i v_j \ra} \cdot \sqrt{\la v_j, \hat S \iv v_j\ra} \\ &= \sqrt{\frac{\la v_j, R_i v_j \ra}{\beta'}},
    \end{aligned}
    \end{equation}
    where we used $T_i \hat S T_i = R_i$ and $\hat S \iv v_j = v_j / \beta'$.
    Averaging and using concavity of the square root,
    \begin{equation}
    \begin{aligned}
        \la v_j, \overline{T} v_j \ra  = \sumin w_i \la v_j, T_i v_j \ra  \leq \frac{1}{\sqrt{\beta'}} \sumin w_i \sqrt{\la v_j, R_i v_j \ra} &\leq \sqrt{\frac{\la v_j, \left(\sumin w_i R_i\right) v_j \ra}{\beta'}} \\ &\leq \sqrt{\frac{\lmax\left(\sumin w_i R_i\right)}{\beta'}} = 1,
    \end{aligned}
    \end{equation}
    the final equality by the definition of $\beta'$; hence $\delta_j (1 - v_j\dg \overline{T} v_j) \geq 0$ as required.
    
    We now consider the floor case.
    For $j$ such that $\delta_j = \lambda_j - \alpha' < 0$, we have $\hat S v_j = \alpha' v_j$.
    Here the required result follows from operator monotonicity of the geometric mean alone.
    Since $R_i \geq \lmin(R_i) I$, operator monotonicity of the geometric mean gives the operator inequality
    \begin{equation}
        T_i = \hat S \iv \# R_i  \geq  \hat S \iv \# \Lt(\lmin(R_i) I\Rt) = \sqrt{\lmin(R_i)}  \hat S \ihf.
    \end{equation}
    Evaluating the quadratic form at $v_j$, where $\hat S \ihf v_j = v_j/\sqrt{\alpha'}$, and averaging,
    \begin{equation}
        \la v_j, \overline{T} v_j \ra = \sumin w_i \la v_j, T_i v_j \ra  \geq \frac{1}{\sqrt{\alpha'}} \sumin w_i \sqrt{\lmin(R_i)} = \frac{\sqrt{\alpha'}}{\sqrt{\alpha'}} = 1,
    \end{equation}
    the final equality by the definition of $\alpha'$; hence $\delta_j(1 - \la v_j, \overline{T} v_j \ra) \geq 0$ in this case as well.
    Every term of~\cref{Eq:RefinedTermwise} is non-negative, so $\la \nabla f_1(\hat S), S - \hat S \ra \geq 0$, completing the proof.
\end{proof}
The second ingredient of the analysis of~\cref{Sec:Convergence} is the strong-convexity modulus feeding the PL inequality, and on the refined interval it cannot be borrowed as stated: the modulus of~\cref{Eq:MuF} is derived for references and iterates sharing the interval $\abi$, whereas here the iterates live in $[\alpha' I, \beta' I]$ while individual $R_i$ may lie outside it.
However, no new work is needed because the strong convexity of~\cref{Eq:MixedSC} is already \emph{mixed}: its constant couples $\lmin(R)$ of the reference with an upper bound $b$ on the argument alone, with no common interval required.
The following lemma records the specialization to $f_1$ and $b = \beta'$, producing the modulus $\mu'$ whose combination with the PL inequality yields the $\kappa'$-rate.
\begin{tbox}
    \begin{lemma}[{specialization of~\cref{Eq:MixedSC}; see~\cite[proof of Thm.~1]{bhatia2018strong}}] \label{Lem:MixedSC}
        The barycenter functional $f_1$ is $\mu'$-SC on $\{S \in \pdd : S \leq \beta' I\}$, where
        \begin{equation}
            \mu' := \frac{\sqrt{\alpha'}}{4 \beta'^{3/2}}, \qquad \text{so that} \qquad 4 \mu' \alpha' = \left(\frac{\alpha'}{\beta'}\right)^{3/2} = \kappa'^{-3/2}, \qquad \kappa' := \frac{\beta'}{\alpha'} \leq \kappa.
        \end{equation}
    \end{lemma}
\end{tbox}
\begin{proof}
    Let $S \leq \beta' I$. Applying~\cref{Eq:MixedSC} to each $\mathrm{B}_{R_i}$ with $b = \beta'$,
    \begin{equation}
        \nabla^2 f_1 (S) = \frac12 \sumin w_i \nabla^2 \mathrm{B}_{R_i} (S)  \geq  \frac{1}{4 {\beta'}^{3/2}} \sumin w_i \sqrt{\lmin(R_i)} \cdot \mathrm{Id} = \frac{\sqrt{\alpha'}}{4 {\beta'}^{3/2}} \mathrm{Id} \equiv \mu' \cdot \mathrm{Id}, 
    \end{equation}
    where $\mathrm{Id} : \ld \to \ld$ is the identity map on the matrix space and the last equality by the definition of $\alpha'$. 
\end{proof}
With the two ingredients in place, the analysis of~\cref{Sec:Convergence} runs on the refined interval verbatim.
\begin{tbox}
    \begin{cor}[Convergence on the refined interval] \label{Cor:RefinedRate}
        Let $S_0 \in [\alpha' I, \beta' I]$ and let the iterates be generated by unit-step BW-GD projected onto the refined interval, $S_{t+1} := \mathrm{clip}_{\alpha', \beta'}[\rmk(S_t)]$.
        Then
        \begin{equation}
            f_1(S_t) - f_1(S_\star) \leq \left(1 - \frac{1}{\kappa'^{3/2}}\right)^t \Lt(f_1(S_0) - f_1(S_\star)\Rt),
        \end{equation}
        the verbatim analogue of~\cref{Thm:UnifiedConvergence} in the refined constants, and the iteration count of~\cref{Eq:IterationCount} improves accordingly to $\kappa'^{3/2} \log(1/\epsilon)$.
    \end{cor}
\end{tbox}
\begin{proof}
    Write $\eps_t := f_1(S_t) - f_1(S_\star)$ and fix $A_t \in \pi \iv [S_t]$; all iterates (by construction of the clipping step) and $S_\star$ lie in $[\alpha' I, \beta' I]$.
    Chaining the unit-step descent inequality (\cref{Prop:VarianceInequalityBuresDistance}) with~\cref{Prop:RefinedMonotone}, and using the step-length identity $\rb(S_t, \rmk(S_t)) = \|\nabla g_1(A_t)\|_\rf^2$ (\cref{Lem:StepLength}), gives $\eps_{t+1} \leq \eps_t - \frac12 \|\nabla g_1(A_t)\|_\rf^2$.
    By~\cref{Lem:MixedSC}, $f_1$ is $\mu'$-SC on $\{S \in \pdd : S \leq \beta' I\} \supseteq [\alpha' I, \beta' I]$, so~\cref{Lem:PLIneqfromSC} (applied with $[a I, b I] = [\alpha' I, \beta' I]$ and $\mu_0 = \mu'$) gives $\kappa'^{-3/2} \eps_t = 4 \mu' \alpha' \eps_t \leq \frac12 \|\nabla g_1(A_t)\|_\rf^2$.
    Combining the two and unrolling yields the rate; the iteration count follows as in~\cref{Sec:Certificate}.
\end{proof}
The canonical initialization $\frac{\alpha + \beta}{2} I$ of~\cref{alg:projected_bwgd} may lie outside the refined interval; one clips it onto $[\alpha' I, \beta' I]$ first, or initializes at $\frac{\alpha' + \beta'}{2} I$.
We note that the statement is specific to the barycenter problem: for the invariant projection problem (a single reference $R$), the refined constants reduce to $\alpha' = \lmin(R) = \alpha$ and $\beta' = \lmax(R) = \beta$, and nothing is gained.
Moreover, in the worst case over ensembles, $\alpha' = \alpha$ and $\beta' = \beta$ are attained, so the refinement improves the constants on favorable instances while leaving the worst-case rate of~\cref{Thm:UnifiedConvergence} unchanged; on every instance, the count $\kappa'^{3/2} \log(1/\epsilon)$ is expressed in a condition number no larger than that of the $O(\kappa'^{5/2} \log(1/\epsilon))$ small-step guarantee of~\cite{altschuler2021averaging}, which it therefore dominates.

\section{The Projected Iteration on the Total Manifold} \label{App:TotalManifold}

The base--total equivalence yields a closed-form implementation of both algorithms of~\cref{Sec:Algorithms}, which we now describe, beginning by transporting the entire iteration---projection included---to the total manifold.

\begin{tbox}
    \begin{cor}[Singular-value clipping is Frobenius projection] \label{Cor:FrobeniusProjection}
        For every $A \in \gld$, singular-value clipping $\mathrm{clip}^{\mathrm{sv}}_{\sqrt{\alpha}, \sqrt{\beta}}$ is the unique Frobenius projection onto the \textit{square-roots} of the well-conditioned set:
        \begin{equation}
            \mathrm{clip}^{\mathrm{sv}}_{\sqrt\alpha, \sqrt\beta}(A) = \argmin_{B \in \pi \iv [\abi]} \|A - B\|_\rf.
        \end{equation}
        Moreover, the two projections are intertwined by the quotient map:
        \begin{equation}
            \pi \circ \mathrm{clip}^{\mathrm{sv}}_{\sqrt\alpha, \sqrt\beta} = \clab \circ \pi.
        \end{equation}
    \end{cor}
\end{tbox}
\begin{proof}
    Results of this type are well known for spectral sets: the Frobenius projection onto a set defined by symmetric constraints on the eigenvalues or singular values reduces to projecting the spectrum itself onto the corresponding set of scalars~\cite{lewis2008alternating}.
    Rather than specializing the general theory, we give a short proof, deriving uniqueness from the Projection Lemma through the quotient structure.
    Recall that $\lambda_i(BB\dg) = \sigma_i(B)^2$ for any $B \in \ld$.
    Fix a singular value decomposition $A = U \Sigma V\dg$ with $\sigma_1(A) \geq \cdots \geq \sigma_d(A) > 0$, and write $\hat A \equiv \mathrm{clip}^{\mathrm{sv}}_{\sqrt\alpha, \sqrt\beta}(A) = U\, \mathrm{clip}_{\sqrt\alpha, \sqrt\beta}(\Sigma)\, V\dg$, which is feasible.

    We appeal to the fibre geometry of the total manifold.
    Let $S := \pi(A)$ and $\hat S := \pi(\hat A) = U\, \mathrm{clip}_{\sqrt\alpha, \sqrt\beta}(\Sigma)^2\, U\dg = \clab(S)$---which is precisely the claimed intertwining identity.
    Moreover, $(A, \hat A)$ is an optimal pairing for $(S, \hat S)$: the two commute, so $\mathrm{F}(S, \hat S) = \tr[S \hf \hat S \hf] = \sumid \sigma_i(A)\, \mathrm{clip}_{\sqrt\alpha, \sqrt\beta}(\sigma_i(A)) = \Re \la A, \hat A \ra$, whence $\|A - \hat A\|_\rf^2 = \rb(S, \hat S)$.
    Now let $B$ be any feasible point, with $Q := \pi(B) \in \abi$.
    Then
    \begin{equation} \label{Eq:FrobProjChain}
        \|A - B\|_\rf^2 \geq \min_{B' \in \pi \iv [Q]} \|A - B'\|_\rf^2 = \rb(S, Q) \geq \rb(S, \hat S) = \|A - \hat A\|_\rf^2.
    \end{equation}
    The first inequality holds simply because $B$ lies in the fibre of $Q$.
    The equality that follows is the quotient formula~\cref{Eq:BuresMin}: the Frobenius distance is invariant under right multiplication by unitaries, so the minimum over both fibres is attained already from the fixed representative $A$.
    The second inequality is the Projection Lemma (\cref{Lem:ProjectionLemma}), as $\hat S = \clab(S)$; the final equality is the optimal pairing established above.
    In particular, since $\hat A$ is feasible, the chain shows that it is a minimizer.
    If $B$ is a minimizer, then $\|A - B\|_\rf^2 = \|A - \hat A\|_\rf^2$, so every step in~\cref{Eq:FrobProjChain} holds with equality.
    Equality in the Projection Lemma step forces $Q = \hat S$, since the BW projection is unique (\cref{Lem:ProjectionLemma}); thus $B$ lies in the fibre $\pi \iv [\hat S]$, and equality in the first step says $B$ is a nearest point of this fibre to $A$.
    But the nearest point of a fibre to a full-rank $A$ is unique---it is $\hat S \hf\, \mathrm{Pol}(\hat S \hf A)$, from the optimal unitary of~\cref{Eq:BuresMin} (see also~\cref{Prop:GradientPolarForm} below)---and $\hat A$ is one such point; hence $B = \hat A$.
\end{proof}
\begin{remark}[Projection on the total manifold can be expansive] \label{Rem:NoFrobeniusNonExpansive}
    Unlike with BW projection in the base manifold, the Frobenius projection of~\cref{Cor:FrobeniusProjection} is \emph{not} non-expansive: already for $d = 1$ and $\sqrt\alpha = 1$, the points $\pm\eps$ project to $\pm 1$, an expansion by the factor $1/\eps$.
    The culprit is the floor---$\{B : \sigma_{\min}(B) \geq \sqrt\alpha\}$ excludes a neighborhood of the singular cone and is non-convex, whereas the ceiling alone is the operator-norm ball, whose projection (the upper clip) is non-expansive by convexity.
    This does not contradict~\cref{Lem:ProjectionLemma}: the expansion occurs along the fibre direction ($\pm\eps$ lie in the same fibre), which the quotient metric collapses. 
    Moving to the BW distance is precisely what restores non-expansiveness.
\end{remark}

\begin{tbox}
    \begin{theorem}[Projected BW-GD is projected Euclidean GD] \label{Thm:ProjectedGDEquivalence}
        Let $f : \pdd \to \bbr$ be differentiable, $g := f \circ \pi$, and $(\eta_t)_{t \geq 0}$ a sequence of step sizes.
        Let $(S_t)_{t \geq 0}$ be generated by Projected BW-GD from $S_0 \in \abi$,
        \begin{equation} \label{Eq:ProjectedBaseRecursion}
            S_{t+1} = \clab\Big[ \Lt[I - 2 \eta_t \nabla f(S_t)\Rt] \, S_t \, \Lt[I - 2 \eta_t \nabla f(S_t)\Rt] \Big],
        \end{equation}
        and $(A_t)_{t \geq 0}$ by singular-value clipped Euclidean GD from any $A_0 \in \pi \iv [S_0]$,
        \begin{equation} \label{Eq:ProjectedTotalRecursion}
            A_{t+1} = \mathrm{clip}^{\mathrm{sv}}_{\sqrt{\alpha}, \sqrt{\beta}}\Lt[A_t - \eta_t \nabla g(A_t)\Rt].
        \end{equation}
        Then $\pi[A_t] = S_t$ for every $t \geq 0$.
        Consequently one may propagate $(A_t)_{t \geq 0}$ alone, for $f \in \{f_1, f_2\}$, halting once $\|\hat\nabla g(A_t)\|_\rf^2 \leq 2 \kappa^{-3/2} \delta$ certifies $f(S_t) - f(S_\star) \leq \delta$ (\cref{Cor:Certificate}), and return $S_t = \pi(A_t)$ only upon termination.
    \end{theorem}
\end{tbox}
\begin{proof}
    Induct on $t$; the case $t = 0$ holds by choice of $A_0$.
    Suppose $\pi[A_t] = S_t$.
    By~\cref{Lem:StepLength} the unprojected step $A_t' := A_t - \eta_t \nabla g(A_t)$ satisfies $\pi[A_t'] = [I - 2\eta_t \nabla f(S_t)] S_t [I - 2 \eta_t \nabla f(S_t)]$, the argument of the clip in~\cref{Eq:ProjectedBaseRecursion}.
    Singular-value clipping upstairs implements eigenvalue clipping downstairs by the intertwining identity $\pi \circ \mathrm{clip}^{\mathrm{sv}}_{\sqrt \alpha, \sqrt \beta} = \clab \circ \pi$ of~\cref{Cor:FrobeniusProjection}.
    Hence $\pi[A_{t+1}] = S_{t+1}$.
\end{proof}
For the invariant projection problem ($f = f_2$), both recursions are run with the twirled gradients---$\Phi(\nabla f_2)$ in place of $\nabla f$ in~\cref{Eq:ProjectedBaseRecursion} and $\hat\nabla g$ in place of $\nabla g$ in~\cref{Eq:ProjectedTotalRecursion}---and the theorem holds verbatim: the proof uses only that the step has the form covered by~\cref{Lem:StepLength}, which accepts any Hermitian $X$, and that clipping intertwines with $\pi$, neither of which is affected by the twirl.

This perspective does not lead to a computational advantage: evaluating $\nabla g(A_t) = 2 \nabla f(S_t) A_t$ still passes through $S_t = A_t A_t\dg$ and its square roots.
Its value is structural, in two respects.
First, \cref{Thm:ProjectedGDEquivalence} exhibits the algorithm as projected Euclidean gradient descent over the non-convex spectral set $\pi \iv [\abi]$, the flat-space description under which the PL machinery of~\cref{Sec:Convergence} applies.
Second, the gradient step acquires a closed form with a geometric meaning---it replaces the current square root by the weighted average of the optimally paired square roots of the references---which is exactly the form in which the unit-step descent lemma becomes the variance identity of~\cref{App:UsefulLemmas}; the following proposition formalizes it.
Write $\mathrm{Pol}(X)$ for the unitary polar factor of an invertible $X$, so $X = \mathrm{Pol}(X) \, |X|$ with $|X| \equiv (X^\dagger X)\hf$.
\begin{tbox}
    \begin{prop} \label{Prop:GradientPolarForm}
        Let $R \in \pdd$, $f_2 = \frac12 \rb(R, \cdot)$, and $g_2 = f_2 \circ \pi$. For every $A \in \gld$,
        \begin{equation}
            A - \nabla g_2(A) = R \hf U = \argmin_{B \in \pi \iv [R]} \|A - B\|^2_\mathrm{F}; \qquad  U \equiv \mathrm{Pol}(R \hf A).
        \end{equation}
    \end{prop}
\end{tbox}
\begin{proof}
    We first establish the second equality.
    It suffices to show $(A, R \hf U)$ is an optimal pairing for $(S, R)$, $S \equiv A A \dg$, i.e.\ $\la R\hf U, A \ra  = \mathrm{F}(R, S)$.
    From $U = \mathrm{Pol}(R \hf A)$,
    \begin{equation}
        R \hf A = U |R \hf A|  \Rightarrow U \dg R \hf A  = \sqrt{A \dg R A} \Rightarrow \la R \hf U, A \ra = \tr [U \dg R \hf A ] = \tr\Lt[\sqrt{A \dg R A}\Rt] = \mathrm{F}(S, R).
    \end{equation}
    Now consider the first equality.
    We claim
    \begin{equation}
        \nabla g_2(A) = 2 \nabla f_2(S) A = [I - R \# S \iv] A = A - R \hf U,
    \end{equation}
    for which we verify the identity $R \# S \iv = R \hf V S \ihf$ with $V \equiv \mathrm{Pol}(R \hf S \hf)$.
    Since $(R \hf S \hf)\dg (R \hf S \hf) = S \hf R S \hf$, the polar decomposition reads $R \hf S \hf = V (S \hf R S \hf)\hf$, so $(S \hf R S \hf)\hf = V \dg R \hf S \hf$ and
    \begin{equation}
        S \iv \# R = S \ihf (S\hf R S \hf)\hf S \ihf = S \ihf V \dg R \hf ,
    \end{equation}
    which equals its adjoint $R \hf V S \ihf$ (as it is Hermitian). 
    Writing $A = S \hf \tilde U$ and using $\mathrm{Pol}(X \tilde U) = \mathrm{Pol}(X) \tilde U$,
    \begin{equation}
        (R \# S \iv) A = R \hf V S \ihf S \hf \tilde U = R \hf \, \mathrm{Pol}(R \hf S \hf \tilde U) = R \hf \, \mathrm{Pol}(R \hf A) = R \hf U. \qedhere
    \end{equation}
\end{proof}

The proposition makes the total-manifold iteration fully explicit, and with it the reading promised in~\cref{Sec:Algorithms}: the update averages optimally paired square roots of the ensemble elements.
Applying~\cref{Prop:GradientPolarForm} termwise to $f_1$, and composing with the twirl in the invariant case, the algorithms become
\begin{equation}
    \begin{aligned}
        A_{t+1} &= \mathrm{clip}^{\mathrm{sv}}_{\sqrt\alpha, \sqrt\beta}\Big[ \sumin w_i \, R_i \hf \, \mathrm{Pol}\Lt(R_i \hf A_t\Rt) \Big], \\
        A_{t+1} &= \mathrm{clip}^{\mathrm{sv}}_{\sqrt\alpha, \sqrt\beta}\Big[ \Phi\Big( R \hf \, \mathrm{Pol}\Lt(R \hf A_t\Rt) \Big) \Big],
    \end{aligned}
\end{equation}
for the barycenter and invariant problems respectively, the latter initialized at some $A_0 \in \mathcal{G}$.
The barycenter iteration thus averages the square roots of the ensemble elements that form an optimal pairing with the current iterate.
Writing $A_t' \equiv A_t - \hat\nabla g(A_t)$ for the argument of the clip---the unprojected step---the stopping rule of~\cref{Thm:ProjectedGDEquivalence} reads $\|A_t - A_t'\|_\rf^2 \leq 2 \delta \kappa^{-3/2}$, which is equivalent to requiring the step-length be short.

\section{The BRT Fixed-Point Algorithm as BW Gradient Descent on the Invariant Submanifold} \label{App:BRT_Algorithm}

In this appendix, we show that the fixed-point algorithm of~\textcite{brahmachari2025fixed} is Bures--Wasserstein gradient descent restricted to the invariant submanifold, contextualizing it within Riemannian optimization.
Beyond the contextualization, the payoff is in showing that the algorithm navigates the flat geometry of the invariant square-root subspace, which enables the seamless application of PL inequalities in~\cref{Sec:Convergence}.

\paragraph{The optimization problem and unconstrained gradient.}
Given a unitary group $\mathcal{U}$ and a non-invariant $R \in \pdd$, the objective is to Bures-project $R$ onto $\mathcal{P} \equiv \cmu \cap \pdd$:
\begin{equation}
    f_2(S) = \frac{1}{2} \rb(R, S) \qaq \nabla f_2(S) = \frac{1}{2}(I - S\iv \# R).
\end{equation}
In general the Riemannian gradient $\grad f_2(S) \notin \mathrm{T}_S \clp \cong \mathcal{H}$; to remain on the invariant submanifold, it must be orthogonally projected, with respect to the BW metric, onto $\mathcal{H}$.
The following lemma shows this projection is the twirl.

\begin{tbox}
    \begin{lemma}[Twirling is Projection] \label{Lem:OrthoProj}
        For all $P \in \clp$ and $X \in \tp \pdd \cong \hrd$, the twirl $\Phi(X)$ is the orthogonal projection of $X$ onto $\tp \clp$ with respect to $\langle \cdot, \cdot \rangle_P^\mathrm{BW}$:
        \begin{equation}
            \Phi(X) \in \tp \clp  \cong \mathcal{H} \qaq \la \Phi(X), X - \Phi(X) \ra_P^\mathrm{BW} = 0.
        \end{equation}
    \end{lemma}
\end{tbox}
\begin{proof}
    $\Phi(X) \in \mathcal{H}$ by definition.
    Applying $\Phi$ to the Lyapunov equation $X = P \lp(X) + \lp(X) P$, and using that $P$ commutes with every element of $\mathcal{U}$, yields $\Phi(X) = P \Phi(\lp(X)) + \Phi(\lp(X)) P$, hence $\Phi(\lp(X)) = \lp(\Phi(X))$.
    Consequently
    \begin{equation}
        \langle \Phi(X), X - \Phi(X) \rangle_P^\mathrm{BW} = \frac12 \tr[\lp(\Phi(X)) (X - \Phi(X))] =  \frac{1}{2} \tr [\Phi(\lp(X)) (X - \Phi(X))] = 0,
    \end{equation}
    since $\Phi$ is idempotent and self-adjoint with respect to the HS inner product.
\end{proof}

\paragraph{The invariant submanifold is totally geodesic.}
\cref{Lem:OrthoProj} identifies the Riemannian gradient of $f_2$ restricted to $\clp$ as the twirled gradient, but RGD \emph{on} the submanifold requires the exponential map of $\clp$ with the induced metric.
It turns out one may simply use $\mathrm{Exp}^{\mathrm{BW}}_S$, because $\mathcal{P}$ is a \textit{totally geodesic} submanifold.
Note that it does not suffice to verify that the endpoint of the step lies in $\clp$: a curve may leave the submanifold and return to it, in which case it would not be a genuine geodesic of the submanifold.
Recall that a Riemannian submanifold $\mathcal{N} \subseteq \mathcal{M}$ is \textit{totally geodesic} if every geodesic of $\mathcal{M}$ starting tangent to $\mathcal{N}$ remains in $\mathcal{N}$ throughout its interval of existence.

\begin{tbox}
    \begin{lemma}[The invariant submanifold is totally geodesic] \label{Lem:TotallyGeodesic}
        $\clp$ is a totally geodesic submanifold of the BW manifold; consequently
        \begin{equation}
            \mathrm{Exp}^{\clp}_P = \mathrm{Exp}^{\mathrm{BW}}_P \big|_{\tp \clp}, \qquad \text{for all } P \in \clp,
        \end{equation}
        wherever either side is defined.
    \end{lemma}
\end{tbox}
\begin{proof}
    Let $P \in \clp$ and $X \in \mathcal{H} \cap \mathrm{dom}(\mathrm{Exp}_P^\mathrm{BW})$; recall $X \in \mathrm{dom}(\mathrm{Exp}_P^\mathrm{BW})$ if and only if $I + \lp(X) > 0$~(\cref{Eq:BWExpMap}).
    For any $t$ with $I + t \lp(X) > 0$ (in particular all $t \in [0,1]$, since there $I + t \lp(X) = (1-t) I + t [I + \lp(X)] > 0$) we have
    \begin{equation}
        \begin{aligned}
            \gamma(t) \equiv \mathrm{Exp}^\mathrm{BW}_P [t X]
            &= \pi\Lt[\Lt(I + t \lp(X)\Rt) P \hf\Rt]
            = \pi\Lt[\Lt(I + t \lp(\Phi(X))\Rt) P\hf\Rt]  \\
            &= \pi\Lt[\Lt(I + t \Phi(\lp(X))\Rt) P\hf\Rt]
            = \pi\Lt[\Lt[\Phi(I + t \lp(X))\Rt] P \hf\Rt]
        \end{aligned}
    \end{equation}
    where we used, in order: $\Phi(X) = X$ for invariant $X$; $\Phi \circ \lp = \lp \circ \Phi$ for invariant $P$ (proof of~\cref{Lem:OrthoProj}); unitality of $\Phi$; and closure of the subalgebra $\cmu$ under products.
    Combined with the domain characterization, $\gamma(t) \in \pdd \cap \cmu = \clp$ for every $t$ in the interval of existence, which is precisely the statement that $\clp$ is totally geodesic.
    Consequently the exponential maps agree on $\tp \clp$.
\end{proof}

\paragraph{Deriving the submanifold update rule.}
By~\cref{Lem:OrthoProj,Lem:TotallyGeodesic}, RGD on $\clp$ with step size $\eta$ is
\begin{equation} \label{Eq:SubmanifoldUpdateRule}
    S \mapsto  [I - 2 \eta \Phi(\nabla f_2(S))] S [I - 2 \eta \Phi(\nabla f_2(S))],
\end{equation}
with projected gradient $\Phi(\nabla f_2(S)) = \frac{1}{2}(I - \Phi(S\iv \# R))$.
Indeed, the step along the projected Riemannian gradient is $\mathrm{Exp}^\mathrm{BW}_S[-\eta\, \Phi(\grad f_2(S))] = \pi\Lt[[I - \eta \lp(\Phi(\grad f_2(S)))]\, S \hf\Rt]$, and since $\lp \circ \Phi = \Phi \circ \lp$ for invariant $S$ (proof of~\cref{Lem:OrthoProj}) while $\nabla f_2 = \frac12 \lp(\grad f_2)$, we have $\lp(\Phi(\grad f_2(S))) = \Phi(\lp(\grad f_2(S))) = 2 \Phi(\nabla f_2(S))$, which yields the displayed update.
At $\eta = 1$ the inner factor becomes $\Phi(S\iv \# R)$, recovering the fixed-point map of~\textcite{brahmachari2025fixed}:
\begin{equation} \label{Eq:BRTFixedPoint}
    S \mapsto \Phi(S\iv \# R)\, S\, \Phi(S\iv \# R) \equiv \rmk(S).
\end{equation}

\paragraph{Equivalence in the square-root manifold.}
The update gains further intuition upstairs: it is projected Euclidean GD on the invariant subspace $\mathcal{G} = \gld \cap \cmu$.
Initializing at an invariant square root $S\hf \in \mathcal{G}$, the twirled lifted gradient is $\Phi(\nabla g_2(S\hf)) = 2 \Phi(\nabla f_2(S)) S\hf$ (as $S\hf$ commutes with $\mathcal{U}$), and the projected Euclidean step
\begin{equation}
    A' = S\hf - \eta \Phi(\nabla g_2(S\hf)) = \Lt[ I - 2\eta \Phi(\nabla f_2(S)) \Rt] S\hf
\end{equation}
maps under $\pi$ exactly to~\cref{Eq:SubmanifoldUpdateRule}.

\end{document}